\documentclass[11pt]{article}
\usepackage{graphicx}
\usepackage{natbib}
\usepackage[utf8]{inputenc} %

\usepackage[top=1.2in, right=1.15in, left=1.15in, bottom=1.2in]{geometry}		%

\usepackage{xspace}		%

\usepackage{amsfonts}
\usepackage{amsmath}
\usepackage{amssymb}
\usepackage{amsthm}
\usepackage{mathtools}		%

\mathtoolsset{%
}

\usepackage[utf8]{inputenc}		%
\usepackage[T1]{fontenc}		%

\usepackage[%
cal=cm,
]
{mathalfa}

\usepackage{dsfont}		%

\usepackage[scale=.96]{inconsolata}
\usepackage{acronym}		%
\usepackage[utf8]{inputenc}		%
\usepackage[T1]{fontenc}		%

\usepackage[dvipsnames,svgnames,table]{xcolor}
\colorlet{MyRed}{Crimson!75!Black}
\colorlet{MyGreen}{DarkGreen!80!Black}
\colorlet{MyBlue}{MediumBlue}

\usepackage[showdeletions]{color-edits}		%
\usepackage[normalem]{ulem}		%

\newcommand{\debug}[1]{#1}		%

\usepackage{algorithm}
\usepackage{algorithmic}

\usepackage{hyperref}
\hypersetup{
colorlinks=true,
linktocpage=true,
pdfstartview=FitH,
breaklinks=true,
pdfpagemode=UseNone,
pageanchor=true,
pdfpagemode=UseOutlines,
plainpages=false,
bookmarksnumbered,
bookmarksopen=false,
bookmarksopenlevel=1,
hypertexnames=true,
pdfhighlight=/O,
urlcolor=blue,linkcolor=blue,citecolor=blue,	%
pdftitle={},
pdfauthor={},
pdfsubject={},
pdfkeywords={},
pdfcreator={pdfLaTeX},
pdfproducer={LaTeX with hyperref}
}
\usepackage[showdeletions]{color-edits}		%
\usepackage[normalem]{ulem}		%

\usepackage[sort&compress,capitalize,nameinlink]{cleveref}	

\crefname{assumption}{Assumption}{Assumptions}

\usepackage{thmtools}		%
\usepackage{thm-restate}		%

\usepackage{caption}
\usepackage{subcaption}

\theoremstyle{plain}
\newtheorem{theorem}{Theorem}		%
\newtheorem{corollary}{Corollary}		%
\newtheorem{lemma}{Lemma}		%
\newtheorem{proposition}{Proposition}		%
\newtheorem{fact}{Fact}

\newtheorem{remark}{Remark}

\newtheorem*{corollary*}{Corollary}		%
\newtheorem*{theorem*}{Theorem}

\theoremstyle{definition}
\newtheorem{definition}{Definition}		%
\newtheorem*{definition*}{Definition}		%
\newtheorem*{assumption*}{Assumptions}		%
\newtheorem*{example*}{Example}		%

\DeclarePairedDelimiter{\braces}{\{}{\}}		%
\DeclarePairedDelimiter{\bracks}{[}{]}		%
\DeclarePairedDelimiter{\parens}{(}{)}		%

\DeclarePairedDelimiter{\abs}{\lvert}{\rvert}		%

\DeclarePairedDelimiterX{\setdef}[2]{\{}{\}}{#1:#2}		%
\DeclarePairedDelimiterXPP{\exclude}[1]{\mathopen{}\setminus}{\{}{\}}{}{#1}

\newcommand{\newmacro}[2]{\newcommand{#1}{\debug{#2}}}		%
\newcommand{\R}{\mathbb{R}}		%
\DeclareMathOperator{\kron}{\otimes} %
\DeclareMathOperator{\kar}{\odot} %
\DeclareMathOperator{\had}{\circ}
\DeclareMathOperator{\ex}{\mathbb{E}}		%
\DeclareMathOperator{\prob}{\mathbb{P}}		%

\DeclarePairedDelimiterXPP{\exof}[1]{\ex}{[}{]}{}{%
 #1}

\DeclarePairedDelimiterXPP{\probof}[1]{\prob}{(}{)}{}{%
 #1}

\DeclarePairedDelimiterXPP{\oneof}[1]{\one}{\{}{\}}{}{%
 #1}

\DeclarePairedDelimiter{\norm}{\lVert}{\rVert}		%
\DeclarePairedDelimiterXPP{\dnorm}[1]{}{\lVert}{\rVert}{_{\ast}}{#1}		%

\DeclarePairedDelimiterXPP{\onenorm}[1]{}{\lVert}{\rVert}{_{1}}{#1}		%
\DeclarePairedDelimiterXPP{\twonorm}[1]{}{\lVert}{\rVert}{_{2}}{#1}		%
\DeclarePairedDelimiterXPP{\supnorm}[1]{}{\lVert}{\rVert}{_{\infty}}{#1}		%

\DeclarePairedDelimiterX{\braket}[2]{\langle}{\rangle}{#1\mathopen{}\delimsize\vert\mathopen{}#2}

\DeclarePairedDelimiterX{\inner}[2]{\langle}{\rangle}{#1,#2}		%

\newcommand\bovermat[2]{%
  \makebox[0pt][l]{$\smash{\overbrace{\phantom{%
    \begin{matrix}#2\end{matrix}}}^{\text{#1}}}$}#2}

\newmacro{\mat}{\R} %
\newmacro{\field}{\mathbf{F}} %

\newcommand{\id}{\mathbb{I}} %
\newcommand{\zero}{\mbox{\normalfont\large 0}}
\newcommand{\rvline}{\hspace*{-\arraycolsep}\vline\hspace*{-\arraycolsep}} %

\newcommand{\vecd}{\mathrm{vecd}}
\newcommand{\myvec}{\mathrm{vec}}
\newcommand{\cf}{\emph{cf.}\xspace}		%
\newcommand{\eg}{\emph{e.g.},\xspace}		%
\newcommand{\ie}{\emph{i.e.},\xspace}		%

\newmacro{\dep}{l} %
\newmacro{\wid}{d} %
\newmacro{\mrank}{r} %
\newmacro{\weights}{\rmW} %
\newmacro{\weight}{w} %
\newmacro{\altweights}{\rmV} %
\newmacro{\altweight}{v} %
\newcommand{\im}[1][{}]{\rmW^{*}_{#1}} %
\newmacro{\iw}{w^*}

\newmacro{\data}{\rvx} %
\newmacro{\altdata}{\rvz}
\newmacro{\dd}{d} %

\newmacro{\batchscale}{\gamma} %
\newmacro{\batchshift}{\beta} %
\newmacro{\mscale}{\rmGamma} %
\newmacro{\mshift}{\bm{\beta}} %

\newmacro{\var}{s} %
\newmacro{\mean}{\bm{\mu}} %

\newmacro{\neuron}{i} %
\newmacro{\altneuron}{j} %

\newmacro{\relu}{{\sigma}} %

\newmacro{\normal}{\matcal{N}} %
\newmacro{\uni}{\mathcal{U}} %

\usepackage{bm}
\usepackage{bbm}
\usepackage[mathscr]{eucal}

\def\1{\bm{1}}

\def\eps{{\epsilon}}

\def\ru{{\textnormal{u}}}
\def\rv{{\textnormal{v}}}

\def\rx{{\textnormal{x}}}

\def\rva{{\mathbf{a}}}
\def\rvb{{\mathbf{b}}}

\def\rvv{{\mathbf{v}}}
\def\rvw{{\mathbf{w}}}
\def\rvx{{\mathbf{x}}}

\def\rvz{{\mathbf{z}}}

\def\rmA{{\mathbf{A}}}
\def\rmB{{\mathbf{B}}}
\def\rmC{{\mathbf{C}}}
\def\rmD{{\mathbf{D}}}
\def\rmE{{\mathbf{E}}}

\def\rmI{{\mathbf{I}}}

\def\rmM{{\mathbf{M}}}

\def\rmP{{\mathbf{P}}}
\def\rmQ{{\mathbf{Q}}}

\def\rmS{{\mathbf{S}}}

\def\rmU{{\mathbf{U}}}
\def\rmV{{\mathbf{V}}}
\def\rmW{{\mathbf{W}}}
\def\rmX{{\mathbf{X}}}

\def\rmGamma{{\mathbf{\Gamma}}}
\def\ermA{{\textnormal{A}}}
\def\ermB{{\textnormal{B}}}

\def\ermP{{\textnormal{P}}}
\def\ermQ{{\textnormal{Q}}}

\def\vmu{{\bm{\mu}}}

\def\vb{{\bm{b}}}

\def\vx{{\bm{x}}}

\def\eva{{a}}

\def\mSigma{{\bm{\Sigma}}}

\DeclareMathAlphabet{\mathsfit}{\encodingdefault}{\sfdefault}{m}{sl}
\SetMathAlphabet{\mathsfit}{bold}{\encodingdefault}{\sfdefault}{bx}{n}

\def\cA{{\mathcal{A}}}

\def\cI{{\mathcal{I}}}
\def\cJ{{\mathcal{J}}}

\def\cO{{\mathcal{O}}}

\def\cS{{\mathcal{S}}}
\def\cT{{\mathcal{T}}}

\def\bP{{\mathbb{P}}}

\def\bR{{\mathbb{R}}}

\addauthor[Dimitris]{DP}{Magenta}

\addauthor[Shashank]{SR}{MediumBlue}

\def\kr{\kar}

\def\bdet{\text{Det}_\text{Bool}}
\newcommand{\Mod}[1]{\ (\mathrm{mod}\ #1)}

\addauthor[Angel]{AG}{DarkOrange}

\definecolor{oldlavander}{rgb}{0.8, 0.8, 1.0}
\definecolor{babypink}{rgb}{0.96, 0.76, 0.76}
\definecolor{bubblegum}{rgb}{1, 0.76, 0.85}
\definecolor{magen}{rgb}{0.6, 0.4, 0.8} %
\definecolor{bluebell}{rgb}{0.64, 0.64, 0.82}
\definecolor{languidlavender}{rgb}{0.84, 0.79, 0.87}
\definecolor{lavenderblue}{rgb}{0.8, 0.8, 1.0}
\definecolor{lavender}{rgb}{0.9, 0.9, 0.98}
\newacro{BN}{Batch Normalization}
\newacro{LN}{Layer Normalization}
\newacro{WN}{Weight Normalization}
\newacro{BNs}{BatchNorm} %

\newacro{NZ}{normalization}

\newacro{NZl}{normalization layer}
\newacro{NZls}{normalization layers}

\newacro{LNs}{LayerNorm}
\newacro{WNs}{WeightNorm}

\newacro{Norm}{Normalization}
\newacro{relu}{ReLU}

\newmacro{\realiz}{R} %
\newcommand{\rank}{\mathrm{rank}}
\usepackage[skins]{tcolorbox}
\newmacro{\chunk}{k}
\usepackage{wrapfig}
\newmacro{\ProjInput}{\rmA}
\newmacro{\ProjLayer}{\rmB}
\newmacro{\layer}{i}
\newmacro{\slayers}{\left\lceil \frac{\wid}{\chunk}\right\rceil} %
\newmacro{\blayers}{\lceil \dfrac{\wid}{\chunk}\rceil} %
\renewcommand{\det}{\text{Det}}
\newcommand{\rankk}{r}
\newmacro{\Subsel}{\rmS} %

\title{\bf The Expressive Power of Tuning Only the Normalization Layers}
\usepackage{times}

\author{
\Large Angeliki Giannou\thanks{Equal contribution, listed alphabetically. Correspondence to \emph{giannou@wisc.edu, rajput.shashank11@gmail.com}.\\Accepted for presentation at the Conference on Learning Theory (COLT) 2023.
}\;, 
Shashank Rajput$^{*}$,
 Dimitris Papailiopoulos \vspace{0.5cm}\\
\Large University of Wisconsin-Madison \\ 
}

\date{}
\begin{document}

\maketitle

\begin{abstract}

Feature normalization transforms such as Batch and Layer-Normalization have become indispensable ingredients of state-of-the-art deep neural networks. 
Recent studies on fine-tuning large pretrained models indicate that just tuning the parameters of these affine transforms can achieve high accuracy for downstream tasks.
These findings open the questions about the expressive power of tuning  the normalization layers of frozen networks.
In this work, we take the first step towards this question and show that for random ReLU networks, fine-tuning only its normalization layers can reconstruct  any target network that is 
$O(\sqrt{\text{width}})$ times smaller.
We show that this holds even for randomly sparsified networks, under sufficient overparameterization, in agreement with prior empirical work.

\end{abstract}

\section{Introduction}
\label{sec: intro}
Many modern machine learning techniques work by training or tuning only a small part of a pretrained network, rather than training all the weights from scratch. This is particularly useful in tasks like transfer learning \citep{ yosinski2014transferable, donahue2014decaf, guo2020parameter, Houlsby3032transfer, zaken2021bitfit}, multitask learning \citep{mudrakarta2018k,clark2019bam}, and few-shot learning \citep{lifchitz2019dense}. Fine-tuning only a subset of the parameters of a large-scale model allows not only significantly faster training, but can sometimes lead to 
better accuracy than training from scratch \citep{zaken2021bitfit,bilen2017universal}.

One particular way of model fine-tuning is to train only the \acl{BN} (\acl{BNs}) or \acl{LN} (\acl{LNs}) parameters
\citep{mudrakarta2018k,huang2017arbitrary}. These \acl{NZls} typically operate as affine transformations of each activation output, and as such one would expect that their expressive power is small, especially in comparison to tuning the weight matrices of a network. However,  in an extensive experimental study, \cite{frankle2020training} discovered that  training these \acl{NZ} parameters in isolation leads to predictive models with accuracy far above random guessing, even when all model weights are frozen at random values. The authors further showed that increasing the width/depth of these random networks allowed \acl{NZls} training to reach significant accuracy across CIFAR-10 and ImageNet, and higher in comparison to training subsets of the network with similar number of parameters as that of \acl{NZls}.

The above experimental studies indicate that training only the normalization layers of a network, seems ---at least in practice--- expressive enough to allow non-trivial accuracy for a variety of target tasks. In this work, we make a first step towards theoretically exploring the above phenomenon, and attempt to tackle the following open question:
\begin{quotation}
 \begin{center}
    \textit{ What is the expressive power of tuning only the normalization layers  of a neural network?}
 \end{center}
\end{quotation}

At first glance, it does not seem that training only the \acl{NZls} has large expressive power. This is because, as we argue in Section \ref{sec:prelim}, at its core, training the \acl{NZl} parameters is equivalent to scaling the input of each neuron and adding a bias term. Let $\rmW$ be a $d\times d$ weight matrix of a given layer, $\bm{\Gamma}$ a diagonal matrix (meaning that has everywhere zero values except for the ones in the diagonal) that normalizes the activation outputs and $\bm{\beta}$  a $d$-dimensional vector that acts as a bias correction term. Then we have that the output of such a layer is equal to

\begin{align}
    f_{\bm{\Gamma}, \bm{\beta}}(\vx)
    &=\sigma\left(\bm{\Gamma} \rmW \vx + \bm{\beta} \right),\label{eq:scaling_transform}
\end{align}
where $\sigma$ is an element-wise activation and $\vx\in\R^d$ is the input to that layer. 
 
Tuning just these two sets of parameters, \ie $\bm{\Gamma}$ and $\bm{\beta}$, seems way less expressive than training $\rmW$, which contains a factor of $\mathcal{O}(d)$ more trainable parameters. 
An even more striking expressive disadvantage of the normalization layer parameters is that they do not linearly combine the individual coordinates of $\vx$. 
A normalization layer only scales and shifts a layer's input. How is it then possible to get high accuracy results by training these simple affine transformations of internal features?

\paragraph{Our contributions. }
In this paper, we theoretically investigate the expressive power of the normalization layers. %
In particular, we prove that any given neural network can be perfectly reconstructed by only tuning the normalization layers of a wider, or deeper random network that contains only a factor of $\widetilde{\cO{}}(\sqrt{\wid})$ more parameters (including both trainable and random).

\begin{figure}
    \centering
  \includegraphics[width=\textwidth]{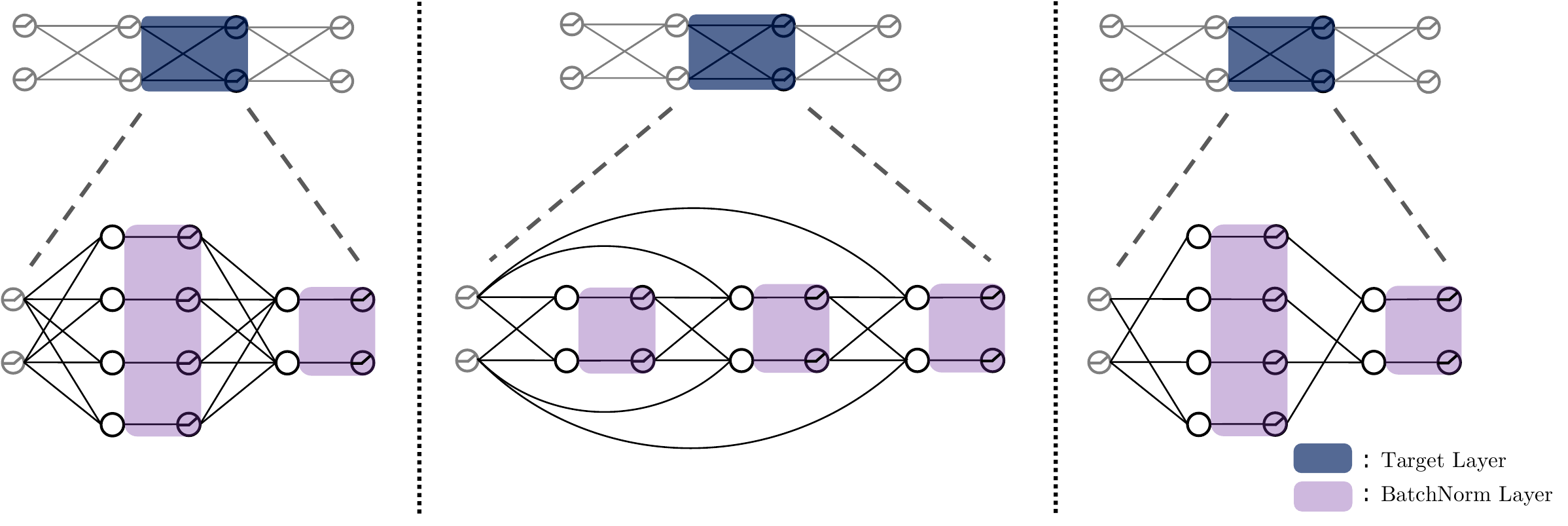}
    \caption{Each fully connected layer of a target network can be exactly recovered by either: (left) two polynomially wider layers, whose only \acl{NZ} parameters are tunable or (middle) a polynomially deeper layer with the same width and skip connections, whose only \acl{NZ} parameters are tunable. We also show that either of these two cases can also work when the random network has sparse weight matrices (right).}
    \label{fig:main_results}
\end{figure}

\begin{theorem*}[Informal]\label{th: shallow informal}
Let $g$ be any fully connected neural network with $\dep$ layers and width $\wid$. Then, any randomly initialized fully connected network $f$ with $\dep'$ layers and width $\wid'$, with \acl{NZls} can exactly recover the network $g$ functionally, by tuning only the \acl{NZl} parameters, as long as $\wid'\dep'\geq 2\wid^2\dep$,  $\wid'\geq \wid$ and $\dep' \geq 2\dep$. Further, if $f$ has sparse weight matrices, then the total number of parameters (trainable and random) only needs to be a factor of $\widetilde{\cO{}}(\sqrt{\wid})$ larger than the target network.
\end{theorem*}

\begin{remark}
Note that for the case when $d'=\cO(d^2)$ and $l'=\cO(l)$, the number of \textbf{trainable} parameters for both the networks are still of the same order, which is $d^2l$. This is because only the diagonal elements of the \acl{NZl} matrix are trainable, which are of the order of $d'=d^2$. This distinction between the total number of parameters and trainable parameters is particularly important in our work's context because the total number of  parameters in $f$ would usually be of the order $\cO(d^3l)$ (as we will see later), however the number of trainable parameters (that is, just the \acl{NZl} parameters) would still be the same.
\end{remark}
A sketch of our construction of random frozen network with tunable \acl{NZls} that allows this result is shown in \cref{fig:main_results}.
The first construction in the figure reconstructs each layer of the target network using two layers of a wider, frozen random neural network with tunable \acl{NZl} parameters. We %
then show that if skip connections are allowed, %
we can reconstruct the target network by a much narrower yet deeper frozen random network with tunable \acl{NZls}. This indicates that adding the skip connections can potentially increase the expressive power of deep networks with \acl{NZls}.

We provide the formal theorem statements for each of the three cases: wide, deep, and sparse reconstructions in Sections \ref{sec: wide}, \ref{sec: deep}, and \ref{sec:sparsify} respectively. The proofs of the theorems rely crucially on the invertibility of Khatri-Rao products of random (possibly sparse) matrices. Due to the complex structure of the Khatri-Rao product, it introduces dependencies amongst the entries of the matrix and consequently, proving the invertibility becomes challenging; while in the case of reconstruction by a deep neural network, some more complicated arguments are required regarding the multiplication of random matrices. We provide proof sketches in the main text of the paper (Sections \ref{sec: wide}, \ref{sec: deep} and \ref{sec:sparsify}), while the full proofs are deferred to the Appendix.

\subsection{Related Work}
Feature Normalization techniques are widely known to improve generalization performance and accelerate training  in deep neural networks.
The first normalization technique introduced was  \acl{BN} by
\citet{Ioffe2015}, followed by \acl{WN} \citep{salimans2016weight} and \acl{LN} \citep{ba2016layer}. The initial method of \acl{BN}  claimed to reduce the \emph{internal covariate shift}, which \citet{Ioffe2015} define as the change in network parameters of the layers preceding to any given layer. The expectation is that by ensuring that the layer inputs are always zero mean and unit variance, training will be faster, similar to how input normalization typically helps training \citep{lecun2012efficient, wiesler2011convergence}. 

In a subsequent work, \citet{balduzzi2017shattered} showed how  \acl{BNs} leads to different neural activation patterns for different inputs; this might be connected to better generalization \citep{morcos2018importance}. 
The claim concerning the internal covariate shift was later doubted by the experimental study of \citet{santurkar2018does}, in which the authors show that adding noise with non-zero mean after \acl{BN} still leads to fast convergence.
 A potential  explanation  is that the loss landscape becomes smoother, a direction explored by \citet{bjorck2018understanding}, who show that large step sizes lead to large gradient norms if \acl{BN} is not used.   

For learning half-spaces using linear models with infinitely differentiable loss functions, \cite{kohler2019exponential} show that \acl{BN} leads to exponentially fast convergence. They show that it reparameterizes the loss such that the norm and direction component of weight vectors become decoupled, similar to weight normalization \citep{salimans2016weight, gitman2017comparison}. \cite{yang2019mean} use mean-field theory to show how residual connections help stabilize the training for networks with Batch Normalization. \cite{luo2018towards} show that \acl{BNs} can be decomposed into factors that lead to both implicit and explicit regularizations. \cite{balestriero2022batch} show that \acl{BN} modifies the geometry of a network by bringing the hyperplanes defined by the neurons closer to the data points, in an unsupervised way. 

 Tuning only the \acl{BN} layers of a random network is related to learning with random features. 
 The works of \citet{Rahimi2008kernel,Rahimi2008UniformAO, Rahimi2008weighted} focused on the power of random features, both theoretically and empirically in the context of kernel methods and later on training neural networks. \cite{andoni14}  investigate linear combinations  of random features  for approximating bounded degree polynomials with complex initialization, through the lens of two layer \acl{relu} networks. In the same context,  \cite{Bach2017}, \cite{Ji2020} give upper bounds on the width for approximating  Lipschitz functions. There are other multiple works that give upper bounds on the necessary width for approximating  with combinations of random features, different classes of functions \citep{Barron93,KlusowskiBarron2016,Sun2018, hsu2021approximation}.

Another result considering random features concerns the overparameterized setting for two layer neural networks: \emph{Stochastic Gradient Descent} tends to keep the weights of the first layer close to their initialization 
\citep{Du2019, yehudai2019power, Jacot2018neural}. 
Building on this, \cite{yehudai2019power} showed that  $poly(d)-$overparameterized two layer \acl{relu} neural network with random features can learn polynomials by training the second layer with SGD . 

On the negative side, there is a line of research on impossibility results in the approximation power of two layer \acl{relu} neural networks. \cite{Ghorbani2019} showed the limitation in the high dimensional setting for approximating high degree polynomials. \cite{yehudai2019power} showed that in order to approximate a single \acl{relu} neuron with a two layer \acl{relu}  network  exponential overparameterization is needed. However,  \cite{hsu2021approximation} give matching upper and lower bounds on the number of neurons needed to approximate any $L-$Lipschitz function. Using a different approximation scheme they achieve polynomial in the dimension bounds for $1-$ Lipschitz functions.

\section{Preliminaries}
\label{sec:prelim}
\paragraph{Notation.} We use bold capital letters to denote matrices (\eg $\rmV,\rmW,\rmA$); bold lowercase  letters to denote vectors (\eg $\rvv,\rvw,\rva$); $\bR^{n\times m}$ to denote the space of all $n\times m$ real matrices;
 
 $\kron$,$\kar$, and $\had$ to denote the Kronecker, the Khatri-Rao and the Hadamard products respectively (\cf Definitions \ref{def: KR } and \ref{def: had}); and $[i]$ to denote the set of integers $\{1,\dots ,i\}$.

When referring to randomly initialized neural networks, we imply that each weight matrix has independent, identically distributed elements, drawn from any arbitrary continuous and bounded distribution. We consider neural networks with ReLU activations, and denote the activation function by $\relu$. 
\begin{definition}[Equivalence / Realization]

We say that two neural networks $f$ and $g$, possibly with different architectures and/or weights are functionally equivalent on a given domain $\mathcal{D}$, if $\forall \vx\in \mathcal{D}$, $f(\vx) = g(\vx)$. We denote this by $f\equiv g$. 
\end{definition}
This is also the same as saying that $f$ and $g$ have the \emph{same realization} \citep{gribonval2022approximation}, or $f$ realizes $g$ and vice-versa. This work focuses on bounded domains, so $\mathcal{D}=\{\vx | \vx\in \R^d \land \|\vx\|\leq 1\}$ unless stated otherwise.

\paragraph{Linear Algebra.} For convenience, we provide the definitions of Khatri-Rao and the Hadamard products below:

\begin{definition}[Khatri-Rao product]\label{def: KR }
The Khatri-Rao product of two matrices $\rmA\in\bR^{n\times mn},\rmB\in\bR^{m\times mn}$ is defined as the matrix that contains column-wise Kronecker products. Formally,
\begin{equation}
    \rmA\kar\rmB = \begin{bmatrix}
    \rmA_1\kron\rmB_1 &\rmA_2\kron\rmB_2 &\hdots &\rmA_{mn}\kron\rmB_{mn}
    \end{bmatrix}
\end{equation}
where $\rmA\kar\rmB\in\R^{nm\times nm}$ and $\rmA_i,\rmB_i$ denote the $i-th$ column of $\rmA$ and $\rmB$ respectively.
\end{definition}

\begin{definition}[Hadamard Product]\label{def: had}
Let $\rmA\in\bR^{n\times m}$ and $\rmB\in\bR^{n\times m}$; then their Hadamard product is defined as the element-wise product between the two matrices. Formally,
\begin{equation*}
    (\rmA\had\rmB)_{i,j} = (\rmA)_{i,j}(\rmB)_{i,j}
\end{equation*}
\end{definition}

\paragraph{Batch Normalization.} \acl{BN} was introduced by \cite{Ioffe2015}. 
However, over time two major variants have been developed: in the first one, the normalization with batch mean and variance is done before applying the affine transformation with the layer's weight parameters. Using the same notation from \eqref{eq:scaling_transform}, we can write this layer as:
\begin{align*}
    f_{\bm{\Gamma}, \bm{\beta}}(\vx)=\sigma(\rmW ^\top \left(\bm{\Gamma} (\epsilon\rmI + \mSigma_{\text{pre}})^{-1}(\vx -\vmu_{\text{pre}}) +\bm{\beta}\right)+\vb), 
\end{align*}
where the mean $\vmu_{\text{pre}}$ and variance $\mSigma_{\text{pre}}$ are those of the raw inputs to the layer $\vx$. Further, $\mscale$ is the diagonal matrix containing the \acl{BNs} scaling parameters $\batchscale$ and $\mshift$ is the vector containing the \acl{BNs} shifting parameters $\batchshift$.

In the second one, the normalization with batch mean and variance is done before applying the affine transformation with the layer's weight parameters:
\begin{align}
    f_{\bm{\Gamma}, \bm{\beta}}(\vx)
    &=\sigma\left(\bm{\Gamma} (\epsilon\rmI + \mSigma_{\text{post}})^{-1}(\rmW ^\top \vx +\vb
    -\vmu_{\text{post}}) +\bm{\beta}\right)\label{eq:BN_type2}
\end{align}
Consequently, the mean $\vmu_{\text{post}}$ and variance $\mSigma_{\text{post}}$ here are those of the $(\rmW ^\top \vx +\vb)$. \cite{he2016identity} showed that the latter leads to better accuracy. In this paper, we consider scaling transformations of the \emph{second} type.

\section{Main Results}\label{sec:results}

In this work we study the expressive power of the normalization layers of a frozen or randomly initialized neural network. All of our theorems concern the \emph{ expressive power } of such a neural network and don't involve any training. As a result, the mean and variance (calculated over a batch size) are considered to be constants, as they are at inference time. However, we mention here that if one wanted to consider variable weight matrices, means and variances, then we could imagine a sequence of frames/instances with each instance representing a stage of the training process. Then for each of these instances we could apply our results.

We  assume that there exists some target network $g(\data)$ that has \acl{NZls}; then we prove that there exists a choice of the parameters of the normalization layers of another randomly initialized neural network $f(\data)$, such that $f(\data)\equiv g(\data)$.

Our first theorem states that a randomly initialized \acl{relu} network with a factor of $\wid$ overparameterization can reconstruct the target \acl{relu} network $g$.
\begin{theorem}\label{thm:width}Let $g(\data)= \mscale_{\dep}^*\im[\dep]\relu(\mscale_{\dep-1}^*\im[{\dep-1}]\relu(\hdots\relu(\mscale_{1}^*\im[1]\data + \mshift_1^*)+\mshift_{\dep-1}^*) +\mshift_\dep^*$ be any ReLU network with depth $\dep$, where $\data\in\R^d$, $\mscale_i^*\in\mat^{\wid\times \wid}$ and $\im[i]\in\bR^{\wid\times \wid}$ with $\|\im[i]\|\leq 1$ for all $i=1,\hdots \dep$. Consider a randomly initialized  ReLU network of depth $2\dep$ with  \acl{NZls}:
\begin{equation*}
f(\data) = \mscale_{2\dep}\weights_{2\dep}\relu(\mscale_{2\dep-1}\weights_{2\dep-1}\relu(\hdots\relu(\mscale_1\weights_1\data + \mshift_1))+\mshift_{2\dep-1}) + \mshift_{2\dep} 
\end{equation*}
where $\mscale_i\in\mat^{\wid^2\times\wid^2}$, $\weights_i\in\bR^{d^2\times d}$ for all odd $i\in[2\dep]$ and  $\mscale_i\in\bR^{\wid\times \wid}$, $\weights_i\in\bR^{d\times d^2}$ for all even $i\in[2\dep]$. Then one can compute the \acl{NZl} parameters $\mscale_1\dots \mscale_{2l}$ and $\mshift_1\dots \mshift_{2l}$, so that  with probability 1 the two networks are equivalent, 
 i.e., $f(\data)=g(\data)$ for all $\data\in\R^d$ with $\|\data\|\leq 1$.
\end{theorem}
\vspace{-0.4ex}
This result shows that the expressive power of scaling and shifting transformations of random features 
is indeed non-trivial. Recent work by \cite{Pufferfish2021}, and \cite{ lowrank2022kamalakara} has shown experimentally that the weight matrices of a neural network can be factorized to low rank ones and then trained with little to no harm in the accuracy of the model. While in \cref{thm:width}, $f$ needs a width overparameterization of the order of $d$ as compared to $g$, we show that if the weight matrices of the target neural network are in fact factorized and have ranks $r$ each, then $f$ only needs a width overparameterization of the order $r$ as compared to $g$. 
Consider the network $g'$ which is created by factorizing the weight matrices of $g$. Since this can be viewed as a network width depth $2l$, we can apply \cref{thm:width} on $g'$ to get the following corollary.
\begin{corollary}
Consider a randomly initialized four layer \acl{relu} network  with \acl{NZls} and width $\wid\mrank $; then  
by  appropriately choosing the scaling and shifting parameters of the \acl{NZls}, the network can be functionally equivalent with  any one layer network of the same architecture, width $\wid$ and weight matrix of $\mrank< \wid$ with probability $1$.
\end{corollary}

The results we presented so far require the width of $f$ to be larger than the width of $g$. This logically leads to the following question: \emph{Can making $f$ deeper help reduce its width?} \cite{frankle2020training} showed experimentally that increasing depth, while keeping width fixed, increases the expressive power of \acl{NZl} parameters. Turns out that this is indeed possible by leveraging skip connections. The following theorem states that any fully connected neural network can be realized by a deeper, randomly initialized neural network with skip connections by only tuning its \acl{NZl} parameters.
\vspace{-0.4ex}
\begin{theorem}\label{th: depth}
Consider the following neural network
\begin{equation*}
    g(\data) = \mscale_{\dep}^*\im[l]\relu(\hdots\relu(\mscale_{1}^*\im[1]\data + \mshift_1^*)) +\mshift_\dep^*
\end{equation*}
where $\mscale_i^*\in\mat^{d\times d},\im[i]\in\mat^{d\times d}$ with $\|\im[i]\|\leq 1$ for all $i=1,\hdots,l$. 
Let the  $j(\slayers +1)$-th layer of the neural network $f$ be 
\begin{align*}
    f^j(\altdata) &= \relu(\weights^{j}_{\slayers+1}L^{j}_{\slayers }) \text{ for } j=1,\hdots, l,\\
    \text{where, }&\altdata\in\R^d,\\
    &\weights^j_{\slayers+1}\in\mat^{d\times dk}, \text{ for } j=1,\hdots,l\\
    &L^{j}_1 = \mscale^{j}_1\weights^{j}_1\altdata^{j}_1 + \mshift^{j}_1,\\
    &L^{j}_{i} = \mscale^{j}_i\weights^{j}_i(\relu\parens{L^{j}_{i-1}}+\ProjInput^{j}_i\altdata^{j}_i) + \mshift^{j}_i\;\text{ for }i = 2,\hdots, \slayers\\
    &\weights^j_i,\mscale^j_i\in\mat^{dk\times dk}, \ProjInput_i^j\in\mat^{dk\times k},\text{ and }\mshift^j_i\in\R^{dk}\text{ for all }i=1,\hdots \slayers, j=1,\hdots,l.
\end{align*}
Here the matrices $\weights_i^j, \forall i,j$ are randomly initialized and then frozen. Then,  with probability 1, one can compute \acl{NZl} parameters so that the two networks are functionally equivalent:
\begin{equation}
    g(\data) = f^\dep(f^{\dep-1}(\hdots f^1(\data))), \forall \|\data\|\leq 1
\end{equation}
Here, the parameter $k$ is tunable and can be any integer in  $[d]$. 
\end{theorem}
Note that the parameter $k$ can be used to achieve a trade-off between width and depth. If $k$ is increased, then the width increases, but the depth decreases and vice-versa. Further for the network $f$, note that each matrix in \cref{thm:width} has dimensions ${d^2\times d}$ or ${d\times d^2}$, for a total of $\cO(d^3l)$ parameters in the network, while in \cref{th: depth} we have matrices of dimensions ${d\times dk}$ or ${dk\times dk}$. If we set $\chunk =1$ the width of $f$ matches that of $g$, however it can be checked that the total number of parameters in $f$ still remains $\cO(d^3 l)$, which is the same as that from \cref{thm:width}.
\begin{remark}
 Theorem \ref{thm:width} and Theorem \ref{th: depth} assume that all layers of $g$ have width $d$, the same as input dimension. However, all the results in the paper can be extended to the general case  where all the layers of $g$ can have different widths. We omit that setting for ease of exposition. 
\end{remark}
Our last result concerns the total number of parameters that the model entails. The question we tackle here is the sparsification of the random matrices of the network in \cref{thm:width}.  
Formally,
\begin{theorem}\label{th: sparse} For the setting of \cref{thm:width}, consider that  the random matrices of each of the layers are sparsified with probability $p=\Theta(\sqrt{\log{\wid}/\wid})$, meaning that each of their elements is zero with probability $1-p$. If the depth of the network is polynomial in the input, \ie $\dep = \text{poly}(\wid)$ then the results of \cref{thm:width} hold with probability at least $1-1/d$. 
\end{theorem}
\begin{remark}
This sparsification results in a total number of $\tilde{O}(d^2\sqrt{d} l)$ non-zero parameters in $f$. However, we don't think that this result is tight. We believe that it can be improved to a higher sparsity of $p=\Theta(\log{\wid}/\wid)$, which would result in the total number of non-zero parameters being $\tilde{O}(d^2 l)$. 
\end{remark}

\section{Our Techniques}
In general, the derivation of our results rely heavily on the invertibility of the Khatri-Rao product, the establishment of full-rankness of matrix multiplications and non-degeneracy of them, by exploiting the randomness of the weight matrices.
\subsection{Reconstruction with overparameterization}\label{sec: wide}

 We prove our first result, \cref{thm:width}, by building a layer-by-layer reconstruction of $g$: For the $i$-th layer of $g$, we construct the $2i-1$ and $2i$ layers of $f$. We use the shifting parameters $\mshift_{2i-1}$ of $f$ to activate all the \acl{relu}s of the $2i-1$-th layer of $f$ and then use the shifting parameters of the next layer, $\mshift_{2i}$, to cancel out any extra bias introduced by $\mshift_{2i-1}$. Finally, we use $\mscale_{2i-1}$ to  reconstruct the targeted matrix. $\mscale_{2i}$ could be set to any arbitrary value (full-rank) diagonal matrix, but for the sake of convenience, we just set it to be the identity matrix.

\paragraph{Proof Sketch.} We prove  that  the first layer of $g$ and the first two layers of $f$ are functionally equivalent. Without loss of generality, the same proof can be applied to all subsequent layers of $g$ and $f$. Thus, it is sufficient to show that  $\mscale_2\weights_2\relu(\mscale_1\weights_1\data + \mshift_1) +\mshift_2 = \im[1]\data$ for all $\vx$ with $\|\vx\|\leq 1$.

\paragraph{Step 1.} We set the parameters of  $\mshift_1$ large enough so that all \acl{relu}s of the first layer of $f$ are activated. For this, we only need that the parameters of the first \acl{NZl} satisfy $\batchshift_1^i =\Omega(|\batchscale_1^i|)$\footnote{See the proof of  \cref{th: width app} for the exact constant.}for each neuron $i$. Next, we simply set $\mscale_2 := \id_{d}$ and  $\mshift_2 := - \weights_2\mshift_1$. Substituting these we get that $\mscale_2\weights_2\relu(\mscale_1\weights_1\data + \mshift_1) +\mshift_2 =\weights_2\mscale_1\weights_1\data$. Thus, now we need to show that the system $\weights_2\mscale_1\weights_1\data = \im[1] \data$ has a solution with respect to $\mscale_1$. %
As we will see, this solution will, in fact, be unique.

\paragraph{Step 2.} We use the following lemma to construct a $\mscale_1$ that satisfies the equality above.
\begin{lemma}\label{lem: solution with khatri-rao}
Let $\rmC\in\mat^{n\times m}$, $\rmX\in\mat^{nm\times nm}$ be a diagonal matrix, $\rmB\in\mat^{nm\times m}$ and $\im\in\mat^{n\times m}$. Then, the following  holds for the equalities below
\begin{equation*}
    \rmC\rmX\rmB = \im \iff (\rmC\kar\rmB^T)\mathrm{vecd}(\rmX) = \mathrm{vec}(\im),
\end{equation*}
where $\mathrm{vecd}$ denotes the vectorization of the diagonal of a matrix, and $\mathrm{vec}$ denotes the row major vectorization of a matrix.
\end{lemma}
The lemma above implies that if $\weights_2\kar\weights_1^T$ is invertible, then we can find unique $\mscale_1$\footnote{Since the matrix $\weights_2\kr\weights_1$ is full rank, the vector $\mathrm{vec}(\im)$ will be a linear combination of its columns; from the well-known theorem of Rouch\'e-Kronecker-Capelli (see \cref{thm:Capelli}) the solution will be unique.} that satisfies $\weights_2\mscale_1\weights_1 =\im[1]$. The following lemma says that with probability 1, this is true.
\begin{lemma}\label{lem: KR full rank } Let $\rmA\in\mat^{n\times nm}$, $\rmB\in\mat^{m\times nm}$ be two random matrices, whose elements are drawn independently from a continuous distribution; then, their Khatri-Rao product , $\rmA\kar\rmB$, is full rank with probability $1$.
\end{lemma}

For some intuition regarding the proof of this lemma, consider $\rmA$ to be an $nm\times nm$ matrix, whose elements are drawn from any continuous distribution independently. It is easy to see that this matrix is full rank, since the columns of this matrix are independent vectors, and the entries of vectors themselves come from a product distribution over $\bR^{nm}$. For the case of the Khatri-Rao product $\rmA\kar\rmB$, where $\rmA\in\R^{n\times nm}$ and $\rmB\in\R^{m\times nm}$ however, we only have $(n+m)nm$ `free' random variables, instead of $(nm)^2$. However, writing down the expression for the determinant as a polynomial in the  random variables, we see that certain terms have independent coefficients from others. This helps us prove that the probability of the determinant being zero is zero. 

\paragraph{Step 3.} Repeating the same proof above for all the layers of $g$ and the corresponding layers of $f$ proves a layer-wise equivalence between $f$ and $g$. Since $f$ and $g$ are just the composition of these layers, proving the equivalence layer-wise proves that $f\equiv g$.

\subsection{Width/Depth tradeoff}\label{sec: deep}

The architecture we consider is as follows (see \cref{fig:deep}): Let $\data\in\R^d$ denote the input to the layer $i$ of $g$. Then for constructing the corresponding layers of $f$ which are functionally equivalent to the the layer $i$ in $g$, we first partition $\data\in\R^d$ into blocks of size $\chunk$, which we denote with $\data_1,\hdots,\data_{\slayers}$. 
 Each of these blocks is then passed to a different layer through the skip connections, with each layer having width $\wid\chunk$. 
 \begin{wrapfigure}[14]{r}{0.6\textwidth}
\centering
\includegraphics[trim={16em 20em 40em 18em}, clip, width=\linewidth]{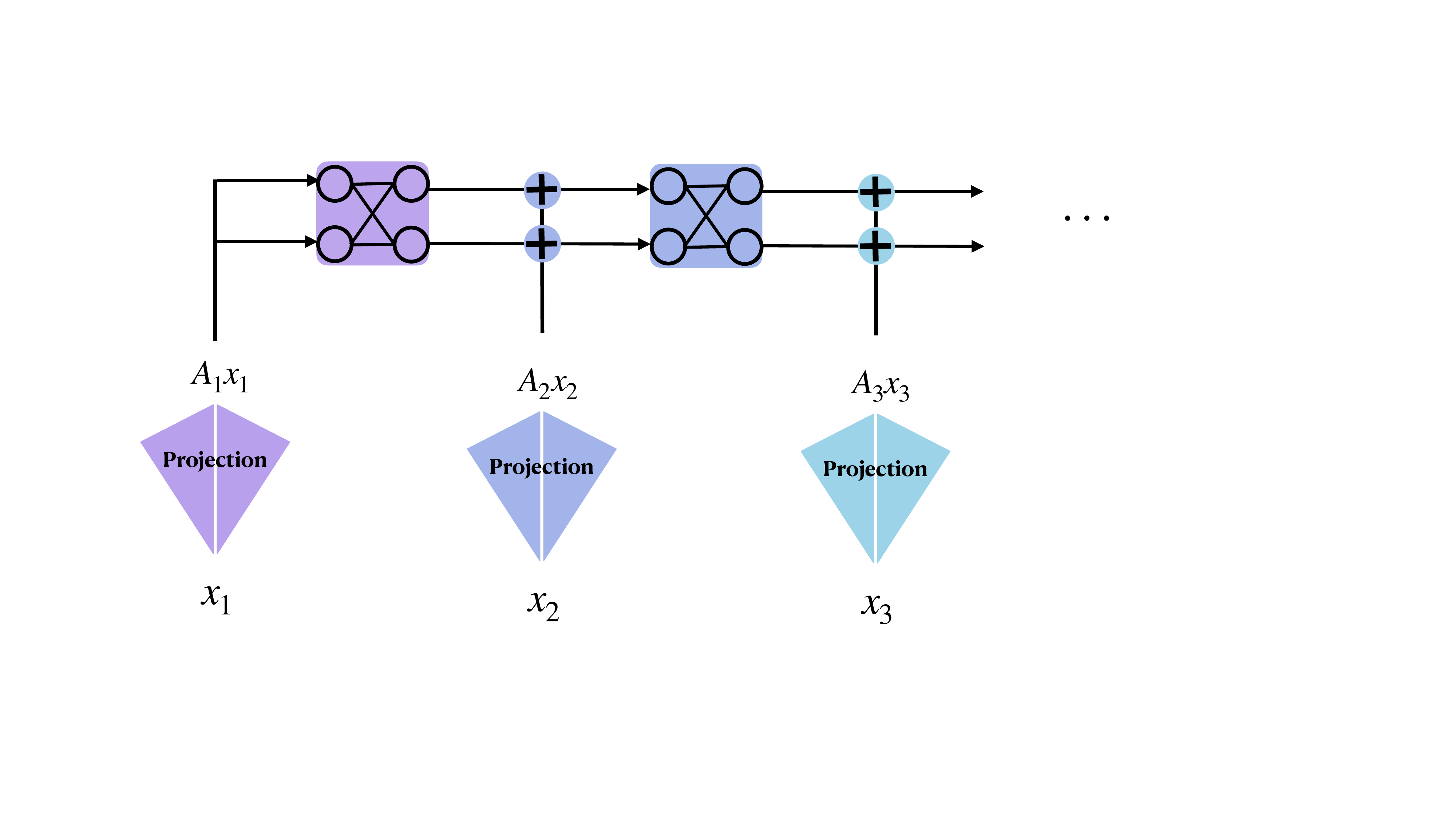}
\caption{A visualization of the skip connections.  $\data_i$ is the $i$-th partition of the input $\data$, which is then projected through the matrices $\rmA_i$ to correct its dimension. }
\label{fig:deep}
\end{wrapfigure}

At the end of $\slayers$ layers, a final linear layer is added to ``correct'' the dimensions and match them 
with the ones of the target network. 
This new architecture is applied to substitute each of the layers of the target network, as is illustrated in \cref{fig:main_results}.

\paragraph{Skip connections.}  As explained earlier, we break the input $\data$ into chunks of size $k$. Thus for example the first chunk, $\data_1 = \begin{bmatrix}
    x_1, \hdots ,x_k
\end{bmatrix}$. This can be seen as a projection to a $k$ dimensional subspace. However, as we mentioned each layer of this construction will have dimension $dk$, we thus use now a random matrix $\rmA_i\in\R^{k\times dk}$, $i=1,\hdots \slayers$  to project the input and match the dimensions.

\paragraph{Proof Sketch.} 
 The idea is to partition the input, pass a different piece of the input in each individual layer and try to exactly reconstruct the parts of the target matrix $\im$ that correspond to the specific piece. Let $\chunk $ be the size of each partition and so in total we need to have $\slayers$ layers. More specifically, we can rewrite $\im$ as

\begin{equation}\label{eq: summation main}
  \vspace{2ex}
  \im  \data= \im \Subsel^{T}_{1}\Subsel_1 \data + \im\Subsel_2^T\Subsel_2\data+\hdots+\im\Subsel_{\slayers}^T \Subsel_{\slayers}\data  
\end{equation}
where $\Subsel_i\in\mat^{k\times d}$ is a sub-selection matrix and specifically
$    \Subsel_i = \begin{bmatrix}
      \zero &\rvline & \id_{\chunk} &\rvline &\zero
    \end{bmatrix}$
where the $k\times k$ identity matrix occupies the $(i-1)\chunk +1, \dots ,i\chunk$ columns. 

Having this in mind, we will use each one of the layers with the skip connections to approximate a specific part of the input, which %
is then projected to a $d\chunk$-dimensional vector  through a random matrix $\ProjInput_i$ %
to match the dimensions of the input. The output of the Neural Network is 
\begin{equation}
    f(\data) = \weights_{\slayers+1}L_{\slayers }
\end{equation}
where $L_1 = \mscale_1\weights_1\data_1 + \mshift_1$ and
$L_{i} = \mscale_i\weights_i[\relu\parens{L_{i-1}}+\ProjInput_i\data_i] + \mshift_i$ for $i = 2,\hdots, \slayers-1$. %

\paragraph{Linearization.} We will set the parameters $\mshift_i$, $i=1,\hdots,\slayers-1$ to be such that all \acl{relu}s are activated as we did in the proof of \cref{thm:width}, while we set $\mshift_{\slayers}$ appropriately to cancel out any error created.

\paragraph{Segmentation.} After the previous step has been completed, the output of the network becomes 
\begin{equation}\label{eq:deep net main}
    f(\data) = \weights_{\slayers+1}\parens*{\prod_{i=1}^{\slayers}\mscale_i\weights_i\ProjInput_1\data_1+ \prod_{i=2}^{\slayers}\mscale_i\weights_i\ProjInput_2\data_2+\hdots +\mscale_{\slayers}\weights_{\slayers}\ProjInput_{\slayers}\data_{\slayers}}
\end{equation}
We now use each of the terms in the summation above to reconstruct the corresponding term in \cref{eq: summation main}. Thus we need to solve a system of $\slayers$ linear systems. 
We show that this system is indeed feasible, by using  induction. 

\paragraph{Induction.} For the base case, we start with the last equation and we show that the linear system
\begin{equation}
    \weights_{\slayers+1}\mscale_{\slayers}\weights_{\slayers}\ProjInput_{\slayers} = \im[\slayers]
\end{equation}
has a unique solution with probability $1$. This is a corollary of the proof of \cref{thm:width}. 

Once the solution for this equation has been found, the idea is to move on to the next layer. One key detail is that we need to prove that all elements of $\mscale_{\slayers}$ will be non-zero with probability one. This is because $\mscale_{\slayers}$ and any $\mscale_i$ appear in \cref{eq:deep net main} the products corresponding to all previous layers. Thus, if one of the elements of these matrices is zero, full-rankness and hence a solution to the corresponding linear systems cannot be guaranteed.
\begin{lemma}\label{lem: non zero component main} Let $\rmA\in\mat^{n\times nm}$, $\rmC\in\mat^{m\times nm}$ be random matrices, $\rmD\in\mat^{nm\times nm}$ a diagonal fixed matrix with non-zero elements and $\iw$ a fixed vector which is  not the zero vector. Define $\rmB = \rmA\kar(\rmC\rmD)$, then
\begin{equation}
    \probof{\inner{\ermB^{-1}_i}{\iw} = 0} = 0
\end{equation}
where $\ermB^{-1}_i$ is the $i-$th row of the inverse of $\rmB$.
\end{lemma}   

For the inductive step, we assume that all components up to (and not included) the $i-th$ component have been exactly reconstructed and $\mscale_j$, $j=\slayers,\hdots, i-1$ have non-zero elements with probability one; and then we reconstruct the $i$-the component and show that $\mscale_i$ also has only non-zero diagonal elements with probability one.

\subsection{Reconstruction of Sparse Networks}\label{sec:sparsify}
   
The proof of this result has two components to it. Our target is to show that if the matrices involved in the Khatri-Rao product appeared in \cref{sec: wide} are sparsified randomly, then with high probability the matrix created is still full rank. We  first show that as long as both matrices have i.i.d. entries from a continuous distribution, the sparsity pattern that is created  dictates whether the matrix is non-singular or not. Then, we prove bounds on the non-singularity of Khatri-Rao products of random Boolean matrices by reducing the problem to the problem of non-singularity of random Boolean matrices with i.i.d. elements.
\paragraph{Proof Sketch.} We now provide  more details of the steps we follow to prove this result. The sparsification of each of the random weight matrices can be expressed as their Hadamard product with a random Bernoulli matrix, \ie $(\weights\had \rmM)$, where $\weights$ has each entry i.i.d. from a continuous distribution and $\rmM$ is a Boolean matrix. From the proof technique presented in \cref{sec: wide}, we conclude that it is sufficient to show that the probability that the matrix $(\weights_1\had\rmM_1)\kar(\weights_2\had\rmM_2)$ is \emph{not} invertible is very small. We assume that the elements of $\rmM_1,\rmM_2$ are \emph{i.i.d.} $\text{Bern}(p)$\footnote{We say that  $x\sim \text{Bern}(p)$, if $\probof{x=1} = p$ and $\probof{x=0} =1-p$ }. Our target is to find a value of the probability $p$ that ensures small enough probability of non-invertibility of the Khatri-Rao product.

\paragraph{Step 1.}  
As mentioned  above, we will relate the invertibility of $(\weights_1\had\rmM_1)\kar(\weights_2\had\rmM_2)$ to the invertibility of $\rmM_1\kar\rmM_2$. To do so, we first introduce the notion of the \emph{Boolean determinant} for a Boolean matrix.
\begin{definition}\label{def: boolean determinant}
The Boolean determinant of a  Boolean matrix $\rmB\in\mat^{d\times d}$ is recursively defined as follows:
\begin{itemize}
\item if $d=1$, $\bdet (\rmB):=\ermB_{1,1}$, that is, it is the only element of the matrix, and
    \item if $d>1$, $
    \bdet (\rmB):=\max_{i\in \{1,\dots, d\}}(\ermB_{i,1}\bdet(\rmB_{-(i,1)}))$.
\end{itemize}    
where $\ermB_{i,j}$ is the element of $\rmB$ at $i$-th row and $j$-th column, and $\rmB_{-(i,j)}$ is the sub-matrix of $\rmB$ with $i$-th row and $j$-th column removed. Note that $\bdet(\rmB)\in \{0,1\}$ for any Boolean matrix $\rmB$. 
\end{definition}
We start by proving the following: $$
 \det\parens{(\weights_1\had\rmM_1)\kar(\weights_2\had\rmM_2)} = 0 \Leftrightarrow \bdet(\rmM_1\kar\rmM_2)=0$$
To prove this, note that despite the fact that the Khatri-Rao product matrix $(\weights_1\had\rmM_1)\kar(\weights_2\had\rmM_2)$ no longer has independent entries, its columns are still independent. Hence, while writing the expression for determinant computed along any column, the corresponding sub-determinants are independent of the column. This combined with the fact that the entries of $\weights_1$ and $\weights_2$ come from continuous distributions, can be used to show the equivalence above.
  
 \paragraph{Step 2.}  Thus, next we need to prove that the Boolean determinant of $\rmM_1\kar\rmM_2$ is zero with a small probability.
 In order to do so, we start replacing each column of $\rmM_1\kar\rmM_2$ with a new vector of equal length, which has $\text{Bern}(q)$ i.i.d. elements. Let this new matrix be $\rmM$. We want to find the value of $q$ as a function of $p$ such that the following holds: $$\probof{\bdet(\rmM_1\kar\rmM_2) =0}\leq \probof{\bdet{\rmM}=0}.$$
 Once we find such a $q$, we can keep replacing each column of $\rmM_1\kar\rmM_2$ to get a matrix $\rmM'$ which all of its entries  are  i.i.d.  $\text{Bern}(q)$ variables, and for which $\probof{\bdet(\rmM_1\kar\rmM_2) =0}\leq \probof{\bdet{\rmM'}=0}$. We can then use existing results on the singularity of Boolean matrices with i.i.d. entries. 

 \paragraph{Step 3.} We show that if $p\geq \sqrt{2qd}$ then the requirement above is satisfied. To prove this, we first show that it is sufficient to prove the following, stronger statement: If $\ru_0,\dots,\ru_{d-1}$, $\rv_0,\dots,\rv_{d-1}\sim \text{Bern}(p)$ and are sampled independently and $\rx_0,\dots,\rx_{d^2-1}\sim \text{Bern}(q)$ are also independent with $p\geq \sqrt{2qd}$; then, for any subset $\cS\subset \{0,\dots, d^2-1\}$,
\begin{align*}
    \bP\left(\sum_{i\in \cS}\ru_{\lfloor i/d \rfloor}\rv_{i\Mod{d}}=0\right)\leq \bP\left(\sum_{i\in \cS}\rx_{i}=0\right)
\end{align*}
This proof requires probabilistic arguments and is deferred to the Appendix. 

\paragraph{Step 4.} 
After this result has been established, we use the following theorem to get bounds on probability of non-singularity of the i.i.d. Boolean entries matrix created in the previous step:
\begin{theorem}[Theorem 1.1 in \citep{Basak2018}]\label{th: Bernoulli}
Let $\rmA\in\mat^{d\times d}$ be an a matrix with \emph{i.i.d.} $Ber(p)$ entries.  Then,  there exist absolute constants $0 < c, \bar{c}, C <\infty$ such that for any $\eps>0$, and $\rho$ such
that $n\rho \geq \log(1/\rho)$, we have
\begin{equation*}
    \prob\parens*{\braces*{\left(s_{min}(\rmA)\leq \bar{c}\eps\exp\left(-C\dfrac{\log(1/\rho)}{\log(n\rho)}\sqrt{\dfrac{\rho}{n}}\right)\right)\cap\Omega_o^c}}\leq \eps + n^{-{c}}
\end{equation*}
where $s_{min}$ is the minimum singular value of the matrix $\rmA$ and $\Omega_0$ is the event that any column or some row is identical to $0$. 
\end{theorem}

It is easy to see that $n\rho\geq \log(1/\rho)$ holds when $\rho\geq \log n/n $.
For the purpose of our work, we will focus on the invertibility of matrix $\rmA$, that is the probablity that $s_{min}(\rmA)>0$. As a result,  let $\eps$  be arbitrarily small  and $\rho = \log n/n$; then the matrix $\rmA$ is non-singular with probability of the order $1/n^c$. We emphasize that the constant $c$ is universal.

Using \cref{th: Bernoulli}, we get that every Bernoulli i.i.d. matrix $\text{Bern}(\rho)$, with $\rho= \log d^2/d^2$, is singular with probability at most $d^{-2c}$ for a universal constant $c$. 
However for our purposes, we need the probability of being singular to be $d^{-\bar{C}}$ for some $\bar{C}$ which we will determine later in the proof. 
To achieve this smaller probability, consider drawing independently $\frac{\bar{C}}{c}$ $\text{Bern}(\rho)$ matrices: $\rmE_i$, $\forall i\in[\frac{\bar{C}}{c}]$ . Then the probability of all of them being singular is at most $d^{-2\bar{C}}$. It is straightforward to see that whenever the determinant of a Boolean matrix is non-zero then the Boolean Determinant of that matrix is also non-zero. Hence, with probability at least $1-d^{-2\bar{C}}$, at least one of these matrices has non-zero Boolean determinant. We create our matrix $\rmM' = \rmE_1 \lor\hdots\lor \rmE_\frac{\bar{C}}{c}$, where the `or' operation is done element-wise. Note that $\rmM' $ has a non-zero Boolean determinant if any of the matrices $ \rmE_1\dots \rmE_\frac{\bar{C}}{c}$ has a  non-zero one. Thus $\rmM'$ has a non-zero Boolean determinant with probability at least $1-d^{-2\bar{C}}$.
Note that $\rmM'$ is itself an i.i.d. Bernoulli matrix, with $q\approx \rho \bar{C}/c = \cO(\log d^2/d^2)$.

\paragraph{Step 5.} To conclude, since $\rmM'$ has a zero Boolean determinant with probability at most  $d^{-2\bar{C}}$, then to ensure that the Khatri-Rao product of the sparse matrices is invertible for all layers, we need to take a union bound over all the layers. Thus, the probability that our result will not hold is at most $d^{-2\bar{C}}2l$. Since we have assumed that $l=\cO(\text{poly}(d))$, for large enough constant $\bar{C}$, the overall probability of the union bound can be driven down to be at most $1/d$. Recalling that $p=\sqrt{2dq}$, we get that $p=\Theta(\sqrt{\log d/d})$ suffices for the theorem to hold.

\section{Discussion and Future Work}\label{sec: Discussion}
This work focuses on the expressive power of \acl{NZ} parameters; but
although our methods are constructive, it is important to note that the results may not apply to the training process through gradient based optimization algorithms.
It would be  interesting though to see if techniques like SGD can indeed leverage this power to the fullest.
We also note that our approach is limited, in the sense that we aimed for exact reconstruction of the target network, and which necessarily meant that we need  $d^2$ \acl{NZ} parameters per layer. The  regime in which  \acl{NZ} layer is only moderately overparameterized could be explored in future work.
Another interesting direction 
 is the fine-tuning capacity of \acl{NZ} layers, where we are given a network trained on a certain task, and we want to fine-tune it on a slightly different task by only modifying the \acl{NZ} layers. Extending our results to other architectures like CNNs and Transformers, and proving lower bounds or impossibility results about the expressive power of \acl{NZls} are also exciting open problems.

\section*{Acknowledgements}
DP acknowledges the support of an NSF CAREER Award \#1844951, a Sony Faculty Innovation Award, an AFOSR \& AFRL Center of Excellence Award FA9550-18-1-0166, an NSF TRIPODS Award \#1740707, and an ONR Grant No. N00014-
21-1-2806. SR acknowledges the support of a Google PhD Fellowship award.

\bibliographystyle{plainnat}
\bibliography{Bibliography.bib}

\newpage
\appendix

\section{Linear Algebra}\label{app: linear}
In this part we mention some useful properties of the Kronecker and Khatri-Rao products, as well as some basic results from linear algebra.

\subsection{Multiplication of matrices \& full rank}

\begin{fact}\label{fact: multiplication is full rank}
Let $\rmA\in\mat^{n\times m}$ and $\rmB\in\R^{m\times l}$ be two full rank matrices, with $n<m$ and $l<m$. If the nullspace of $\rmA$ doesn't have intersection with the column space of $\rmB$ then $\rank(\rmA\rmB) = \min\braces{n,l}$.
\end{fact}
\begin{remark}
We note here, that the above  is a necessary and sufficient condition for the multiplication of two matrices to be full rank. The same argument can also be made through the left nullspace and the row space.
\end{remark}
\subsection{Some products \& their properties}

\begin{definition}[Kronecker product]\label{def: kroneckerapp}
The Kronecker product of two matrices, which we denote with $\rmA\kron\rmB$, where $\dim(\rmA) = m\times n $ and $\dim(\rmB) = p\times q$ is defined to be the block matrix
\begin{equation}
\rmA\kron\rmB \equiv \begin{bmatrix}
\eva_{1,1}\rmB &\hdots  &\eva_{1,n}\rmB \\ 
 \vdots&\ddots  &\vdots \\ 
 \eva_{m,1}\rmB& \hdots & \eva_{m,n}\rmB
\end{bmatrix}
\end{equation}
\end{definition}
\begin{theorem}\label{th: kronecker}
The Kronecker product has the following properties:
\begin{enumerate}
    \item The rank of the Kronecker product of two matrices is equal to the product of their individual ranks, \ie
\begin{equation}
    \rank(\rmA\kron\rmB) = \rank(\rmA)\rank(\rmB)
\end{equation}
\item Furthermore,
\begin{equation*}
    (\rmA_1\rmA_2\hdots \rmA_n)\kron(\rmB_1\rmB_2\hdots\rmB_n) = (\rmA_1\kron\rmB_1)(\rmA_2\kron\rmB_2)\hdots(\rmA_n\kron\rmB_n)
\end{equation*}
\end{enumerate}
\end{theorem}
Proofs of these properties can be found in \cite{horn_johnson_1991}.
\begin{definition}[Khatri-Rao product ]
The Khatri-Rao product of two matrices $\rmA\in\bR^{n\times mn},\rmB\in\bR^{m\times mn}$ is defined as the matrix that contains column-wise Kronecker products. Formally,
\begin{equation}
    \rmA\kar\rmB = \begin{bmatrix}
    \rmA_1\kron\rmB_1 &\rmA_2\kron\rmB_2 &\hdots &\rmA_{mn}\kron\rmB_{mn}
    \end{bmatrix}
\end{equation}
where $\rmA_i,\rmB_i$ denote the $i-th$ column of $\rmA$ and $\rmB$ respectively.
\end{definition}

\begin{definition}[Hadamard Product]
Let $\rmA\in\bR^{n\times m}$ and $\rmB\in\bR^{n\times m}$; then their Hadamard product is defined as the element-wise product between the two matrices. Formally,
\begin{equation*}
    (\rmA\had\rmB)_{i,j} = (\rmA)_{i,j}(\rmB)_{i,j}
\end{equation*}
\end{definition}

\begin{theorem}[Mixed products]\label{th: mixed products}
The products defined before are connected in the following ways:
\begin{enumerate}
    \item The  Khatri-Rao and the Kronecker product satisfy
\begin{equation}
    (\rmA\rmC)\kar(\rmB\rmD) = (\rmA\kron\rmB)(\rmC\kar\rmD)
\end{equation}
\item The Khatri-Rao and the Hadamard product satisfy
\begin{equation}
    (\rmA\kar\rmB)\had(\rmC\kar\rmD) = (\rmA\had\rmC)\kar(\rmB\had\rmD)
\end{equation}
\end{enumerate}
\end{theorem}
The above two properties can be found in \cite{Liu2008} and \cite{horn_johnson_1991}.
\subsection{Linear Systems}
We state below the well-known theorem of Rouch\'e-Kronecker-Capelli, about the existence of solutions of Linear systems.
\begin{theorem}[Rouch\'e-Kronecker-Capelli]\label{thm:Capelli} Let $\rmA\rvx = \rvb$ be a linear system of equations, with $\rmA\in\R^{m\times n}$, $\rvx\in\R^n$ and $\rvb\in\R^n$. Then this system has
\begin{enumerate}
    \item \emph{No solution:} if $\rank(\rmA)\neq \rank([\rmA|\rvb])$.
    \item \emph{Unique solution:} if  $\rank(\rmA)= \rank([\rmA|\rvb]) = n$.
    \item \emph{Infinite solutions:} if $\rank(\rmA)= \rank([\rmA|\rvb]) < n$.
\end{enumerate}
\end{theorem}
This theorem essentially states that if $\rvb$ is in the column space of $\rmA$, then the system has a solution.

\begin{lemma}\label{lem: solution with khatri-rao app}
Let $\rmC\in\mat^{n\times nm}$, $\rmX\in\mat^{nm\times nm}$ a diagonal matrix, $\rmB\in\mat^{nm\times m}$ and $\im\in\mat^{n\times m}$ then the solution of
\begin{equation}
    \rmC\rmX\rmB = \im
\end{equation} 
with respect to the elements of the matrix $\rmX$ can be recasted as
\begin{equation}
    (\rmC\kar\rmB^T)\mathrm{vecd}(\rmX) = \mathrm{vec}(\im)
\end{equation}
where  we use $\mathrm{vecd}$ to denote the vector consisting of only the diagonal elements of the diagonal matrix $\rmX$ and $\mathrm{vec}$ is the row-major vectorization of $\im$. %
\end{lemma}
\begin{proof}
We want to show that there exists a choice of the elements of $\rmX$ such that
\begin{equation}
    \rmC\rmX\rmB = \im
\end{equation} 
where $\rmX$ is a diagonal matrix. 
Equivalently,  we can rewrite the above equation as
\begin{equation}
\renewcommand\arraystretch{2}
\begin{bmatrix}
\sum_{\neuron} c_{1,\neuron}x_{\neuron}b_{\neuron,1} &\sum_{\neuron} c_{1,\neuron}x_{\neuron}b_{\neuron,2}  &\hdots  &\sum_{\neuron} c_{1,\neuron}x_{\neuron}b_{\neuron,m} \\ 
\sum_{\neuron} c_{2,\neuron}x_{\neuron}b_{\neuron,1} & \sum_{\neuron} c_{2,\neuron}x_{\neuron}b_{\neuron,2} &\hdots  &\sum_{\neuron} c_{2,\neuron}x_{\neuron}b_{\neuron,m} \\ 
 \vdots&  \ddots& \hdots & \vdots \\ 
 \sum_{\neuron} c_{n,\neuron}x_{\neuron}b_{\neuron,1}& \sum_{\neuron} c_{n,\neuron}x_{\neuron}b_{\neuron,2} & \hdots &\sum_{\neuron} c_{n,\neuron}x_{\neuron}b_{\neuron,m} 
\end{bmatrix} = \begin{bmatrix}
\iw_{1,1} &\iw_{1,2} &\hdots &\iw_{1,m}\\
\iw_{2,1} &\iw_{2,2} &\hdots &\iw_{2,m}\\
 \vdots&  \ddots& \hdots & \vdots \\
 \iw_{n,1} &\iw_{n,2} &\hdots &\iw_{n,m}
\end{bmatrix}
\end{equation}
for some choice of the elements of $\rmX$.

We rewrite the above matrix equation as a linear system and we have
\begin{equation}
\renewcommand\arraystretch{1.5}
\left\{\begin{matrix}
\sum_\neuron c_{1,\neuron}x_{\neuron}b_{\neuron,1} &  =&  \iw_{1,1}& \\ 
\sum_\neuron c_{1,\neuron}x_{\neuron}b_{\neuron,2} &  =&\iw_{1,2}  & \\ 
 &  \vdots&  & \\
 \sum_\neuron c_{1,\neuron}x_{\neuron}b_{\neuron,m} &  =&\iw_{1,m}  & \\

  &  \vdots&  & \\
  \sum_\neuron c_{n,\neuron}x_{\neuron}b_{\neuron,1} &  =&  \iw_{n,1}& \\ 
\sum_\neuron c_{n,\neuron}x_{\neuron}b_{\neuron,2} &  =&\iw_{n,2}  & \\ 
 &  \vdots&  & \\
 \sum_\neuron c_{n,\neuron}x_{\neuron}b_{\neuron,m} &  =&\iw_{n,m}  & \\ 
\end{matrix}\right.
\end{equation}
Finally, we rewrite the linear system in a matrix form and we have
\begin{equation}
 \begin{bmatrix}
 c_{1,1}b_{1,1} &c_{1,2}b_{2,1}   &\hdots  &c_{1,nm}b_{nm,1}  \\ 
c_{1,1}b_{1,2} &c_{1,2}b_{2,2}   &\hdots  &c_{1,nm}b_{nm,2}  \\ 
 \vdots&  \ddots& \hdots & \vdots \\ 
 c_{1,1}b_{1,m} &c_{1,2}b_{2,m}   &\hdots  &c_{1,nm}b_{nm,m}  \\ 
 \vdots&  \vdots& \vdots & \vdots \\ 
 c_{n,1}b_{1,1} &c_{n,2}b_{2,1}   &\hdots  &c_{n,nm}b_{nm,1}  \\ 
c_{n,1}b_{1,2} &c_{n,2}b_{2,2}   &\hdots  &c_{n,nm}b_{nm,2}  \\ 
 \vdots&  \ddots& \hdots & \vdots \\ 
 c_{n,1}b_{1,m} &c_{n,2}b_{2,m}   &\hdots  &c_{n,nm}b_{nm,m}
 \end{bmatrix}   \begin{bmatrix}
 x_1\\ 
x_2\\ 
\vdots\\
x_{m}\\
\vdots\\
x_{nm-m+1}\\
x_{nm-m+2}\\
\vdots\\
x_{nm}
\end{bmatrix} = \begin{bmatrix}
 \iw_{1,1}\\ 
\iw_{1,2}\\ 
\vdots\\
\iw_{1,m}\\
\vdots\\
\iw_{n,1}\\
\iw_{n,2}\\
\vdots\\
\iw_{n,m}
\end{bmatrix}
\end{equation}
The matrix appearing in this equation is the so-called Khatri-Rao product (or column-wise Kronecker product) of the $\rmC,\rmB^T$.
\end{proof}

\subsection{Results on random matrices}

\textbf{Khatri-Rao product and rank. }%
The full rankness of the Khatri-Rao for random matrices has been proven in  \cite{JiangSidir2001}, over the set of complex numbers. This however doesn't instantly imply that the result holds for the reals too. We thus prove it here for completeness. We will first introduce a lemma, that  has been proven in \citep{caron2005zero} and shows that if a polynomial is not the zero polynomial then the set that results to zero value of the polynomial is of measure zero. %
\begin{lemma}\label{lem: measure zero}
Let $p(\data)$  be a polynomial of degree $d$, $\data\in\R^n$. If $p$ is not the zero polynomial, then the set
\begin{equation}
    \mathcal{S} := \braces{\data\in\R^n | p(\data) =0}
\end{equation}
is of measure zero.\footnote{Specifically Lebesgue measure zero.}
\end{lemma}

\begin{lemma}\label{lem: KR full rank app}
Let $\rmA,\rmB$ be matrices whose elements are drawn independently from some continuous distributions, with $\rmA\in\mat^{n\times nm}$, $\rmB\in\mat^{m\times nm}$. Then the matrix forming their Khatri-Rao product is full rank with probability $1$.
\end{lemma}
\begin{proof}
We rewrite here the Khatri-Rao product of two matrices for clarity
\begin{equation}
    \rmA\kar\rmB = \begin{bmatrix}
    \eva_{1,1}\ermB_1&\hdots&\eva_{1,nm}\ermB_{nm}\\
    \vdots&\hdots &\vdots\\
    \eva_{n,1}\ermB_1&\hdots&\eva_{n,nm}\ermB_{nm}
    \end{bmatrix}
\end{equation}
Using \cref{lem: measure zero} we have that it is sufficient to find an assignment of the values, such that the determinant of the matrix is not the zero polynomial. One such assignment is the following 
\vspace{3em}
\begin{equation}
\begin{matrix}
    \rmA = \begin{bmatrix}
 \bovermat{m columns}{1& \hdots& 1}&  \bovermat{m columns}{0& \hdots& 0}&  \bovermat{m columns}{0& \hdots& 0}\\
    0& \hdots& 0&  1&\hdots& 1& 0& \hdots& 0 &\hdots\\
    0& \hdots& 0&  0& \hdots& 0 & 1&\hdots& 1 &\hdots\\ 
    \vdots& \ddots& \vdots&  \vdots& \ddots& \vdots&\vdots& \ddots& \vdots &\hdots \\
     0& \hdots& 0  &  0& \hdots& 0  &  0& \hdots& 0 &\hdots
    \end{bmatrix}
\end{matrix}
\end{equation}

And for the choice of B we have 
\begin{equation}
    \begin{matrix}
    \rmB = \begin{bmatrix}
    \id_{m} | \hdots |\id_{m}
    \end{bmatrix}
    \end{matrix}
\end{equation}
Then the Khatri-Rao of $\rmA\kar\rmB$ is equal to the identity matrix and so its determinant is not identical to zero and this completes the proof.
\end{proof}
\begin{remark}\label{rem: KR more matrices}
If we have the Khatri-Rao product of four or three random matrices, \ie $(\rmA\rmB)\kar(\rmC\rmD)$ or $\rmA\kar(\rmC\rmD)$, where $\rmA\in\R^{n\times nm}$, $\rmB\in\R^{nm\times nm}$, $\rmC\in\R^{m\times nm}$ and $\rmD\in\R^{nm\times nm}$ then we can set $\rmD, \rmB$ to be the identity and we just go back to the Khatri-Rao of two matrices.
\end{remark}

\section{Width Overparameterization}\label{app: width}
Before we proceed with the theorems and the proofs of this section we remind the following key concepts \& assumptions used: 
\begin{itemize}
    \item When referring to a randomly initialized neural network with \acl{BNs} it means that all parameters (weights, mean and variance) are randomly initialized from an arbitrary continuous and bounded distribution except for the scaling and shifting parameters which we can control.
    \item Unless stated otherwise $\twonorm{\data}\leq 1$.
    \item We will say that a network has rank $r$ if the rank of the matrix before the activation function is $r$.
\end{itemize}

\begin{tcolorbox}[enhanced,width=5.5in, drop fuzzy shadow southwest,    boxrule=0.4pt,,colframe=yellow!50!blue,colback=blue!10]
We start by showing that almost surely any two layer random \acl{relu} network with \acl{BNs}, of width $\wid^2$ can reconstruct any one layer \acl{relu} network with \acl{BNs} and $\wid$ width, by tuning the normalization scaling and shifting parameters. 
\end{tcolorbox}
\begin{proposition}\label{prop: simple form of one layer}
Any layer of a linear network with \acl{BNs}, \ie 
\begin{equation}
    f(\data) = \mscale\dfrac{\weights \data -\mean}{\var} + \mshift,
\end{equation}
can be equivalently expressed as
\begin{equation}
    f(\data) = \mscale'\weights\data +\mshift'
\end{equation}
\end{proposition}
\begin{proof}
 The proof follows trivially by setting $\mscale' =\mscale/\var$ and $\mshift' = \mshift - \mscale'\mean$
\end{proof}

\begin{lemma}\label{lem: approximation of one layer} Consider a randomly initialized two layer \acl{relu} network 
 $f$ with \acl{BN} and width $\wid^2$; then  
by  appropriately choosing the scaling and shifting parameters of the \acl{BNs}, the network can be functionally equivalent with  any one layer network $h$ of the same architecture and width $\wid$ with probability $1$.
\end{lemma}

\begin{proof}
Using \cref{prop: simple form of one layer}
we can write the one layer \acl{relu} network with \acl{BNs} as
\begin{equation}
   h(\data) = \relu(\mscale^*\im \data + \mshift^*)
\end{equation}
where $\mscale^*,\im\in\R^{\wid\times\wid}$ and $\mshift^*\in\R^{\wid}$.

The output of a  two layer neural network with two layer of \acl{BN}, can be written as
\begin{equation}
   f(\data) = \relu\left(\mscale' \dfrac{\altweights
   \relu\left(\mscale\dfrac{\weights\data -\mean}{s} + \mshift\right) -\mean'}{\var'} +\mshift'\right)
\end{equation}
where $\mscale \in \R^{\wid^2\times \wid^2}$, $\mscale' \in \R^{\wid\times \wid}$ and they are diagonal matrices, $\mshift \in\R^{\wid^2}$ and  $\mshift \in \R^\wid$. Also $\weights\in\R^{\wid^2\times \wid}$ and $\altweights\in\R^{\wid\times \wid^2} $.

The output of the $i$-th neuron in first layer before the \acl{relu} activation is a vector with values
\begin{equation}
   f^1_\neuron(\data) = \frac{1}{\var}\batchscale_\neuron\inner{\weights_{\neuron }}{ \data} -\frac{1}{\var}\batchscale_\neuron\mean_\neuron+ \batchshift_\neuron
\end{equation}
Recall that we assume $\twonorm{\data} \leq 1$, and that the initializations are from a bounded distribution, say $\norm{\weights}\leq C$\footnote{We use here the Frobenius norm.}. Further, let $c = \max\{\abs{\mean_\neuron}\}$ and $c'=\abs{1/\var}$. Then, 
\begin{align*}
-\frac{1}{\var}\batchscale_\neuron\inner{\weights_{\neuron }}{ \data} +\frac{1}{\var}\batchscale_\neuron\mean_\neuron &\leq \abs{\frac{1}{\var}\batchscale_\neuron\inner{\weights_{\neuron }}{ \data} -\frac{1}{\var}\batchscale_\neuron\mean_\neuron} \\
&\leq \abs{\batchscale_\neuron/\var}\abs{\inner{\weights_\neuron}{\data} - \mean_\neuron}\\
&\leq c'\abs{\batchscale_\neuron}(\abs{\inner{\weights_\neuron}{\data}}+\abs{ \mean_\neuron})\\
&\leq c'\abs{\batchscale_\neuron}\parens*{\twonorm{\weights_\neuron}\twonorm{\data} + \abs{\mean_\neuron}}\\ 
&\leq c'(C+c)\abs{\batchscale_\neuron}.
\end{align*}
Hence, if $\batchshift_\neuron \geq C'\abs{\batchscale_\neuron}$
for  $C'=c'(c+C)$,  
then
$\relu(f^1_i(\data)) = f^1_i(\data)$ because $f^1_i(\data)\geq 0$. Then, the output of the Neural Network is
\begin{equation}
   f(\data) =\relu\left(\frac{1}{\var\var'}\mscale' \altweights
   \mscale\weights\data - \frac{1}{\var\var'}\mscale' \altweights
   \mscale\mean + \frac{1}{\var'}\mscale' \altweights\mshift  - \frac{1}{\var'}\mscale'\mean'+\mshift'\right)
\end{equation}
Setting $\mscale' := \var\var'\mscale^*$ and $\mshift' := -\frac{1}{\var'}\mscale' \altweights\mshift + \frac{1}{\var\var'}\mscale'\altweights\mscale\mean + \frac{1}{\var'}\mscale'\mean' + \mshift^*$, we get
\begin{equation}
   f(\data) =\relu\parens{\mscale^*\altweights
   \mscale\weights\data + \mshift^*}
\end{equation}
In order to prove that this Neural Network can exactly express any layer $h(\data) = \relu(\mscale^*\im \data + \mshift^*)$ it is sufficient to show that 
\begin{equation}
   \altweights
   \mscale\weights = \im
\end{equation}
By \cref{lem: solution with khatri-rao app}, combined with \cref{thm:Capelli} it is sufficient to prove that the Khatri-Rao product of $\altweights, \weights^T$ is full rank, then the linear system will have a unique solution. Finally, using \cref{lem: KR full rank app} concludes the proof.
\end{proof}

\begin{remark}
Since we showed in the proof of \cref{lem: approximation of one layer}, how we can incorporate the \acl{BN} parameters of the targeted network in our analysis, we are going to ignore them from now on, for simplicity of exposition.
\end{remark}
Having established this result, we can show that if the targeted neural network is in fact low rank, by doubling the depth, we just need width that scales with the rank. When we say that the neural network is low rank, this can imply two facts one the scaling parameter of the \acl{BNs} has some zeros or that the weight matrix is low rank. In the first case, it follows trivially from the above proof that we will just have to focus on the non-zero rows since we can match the zeros one by choosing the corresponding scaling parameters to be zero. We treat the second case below.
\begin{corollary}
Consider a randomly initialized four layer \acl{relu} network  with \acl{BN} and width $\wid\mrank $; then  
by  appropriately choosing the scaling and shifting parameters of the \acl{BNs}, the network can be functionally equivalent with  any one layer network of the same architecture, width $\wid$ and weight matrix of $\mrank< \wid$ with probability $1$ .
\end{corollary}

\begin{proof}
Since the weight matrix $\im$ has rank $\mrank<\wid$, we can use SVD decomposition to write this matrix as two low rank matrices, \ie
\begin{equation}
    \im = \rmA^*\rmB^*
\end{equation}
where $\rmA^*\in\R^{\wid\times \mrank}$ and $\rmB^*\in\R^{\mrank\times \dd}$. We will only show how to reconstruct this matrix, since the rest of the proof follows as that of \cref{lem: approximation of one layer}. We will use the first two layers to reconstruct  $\rmB^*$  and  the last two layers to  reconstruct $\rmA^*$. More specifically, as we did in the proof of \cref{lem: approximation of one layer}, we first linearize the relu and then we want to show that the following linear system has a solution
\begin{align}
    \altweights_1\mscale_1\weights_1 =\rmB^*
\end{align}
where $\altweights_1\in\R^{r\times dr},\mscale_1\in\R^{dr\times dr}$ and $\weights_1\in\R^{dr\times d}$. This system has a solution as we showed in \cref{lem: approximation of one layer}. We 
now consider as $\data$ the output of these two layers which will be $\rmB^*\data$ after the reconstruction and we use the next two layers to reconstruct $\rmA^*$.
\end{proof}
\begin{remark}
We note here we can also treat $\mscale^*\im$ as one matrix and scale the width w.r.t. the rank of this matrix.
\end{remark}
\begin{theorem}[Restatement of \cref{thm:width}]\label{th: width app}Let $g(\data)= \im[\dep]\relu(\im[{\dep-1}]\relu(\hdots\relu(\im[1]\data))$ be any ReLU network of rank $\rankk$, depth $\dep$, where $\data\in\R^d$ and $\im[i]\in\bR^{\wid\times \wid}$ with $\|\im[i]\|\leq 1$ for all $i=1,\hdots \dep$. Consider a randomly initialized  ReLU network of depth $2\dep$ with  \acl{BN}:
\begin{equation*}
f(\data) = \mscale_{2\dep}\weights_{2\dep}\relu(\mscale_{2\dep-1}\weights_{2\dep-1}\relu(\hdots\relu(\mscale_1\weights_1\data + \mshift_1))+\mshift_{2\dep-1}) + \mshift_{2\dep} 
\end{equation*}
where $\mscale_i\in\mat^{\wid^2\times\wid^2}$, $\weights_i\in\bR^{d\times d^2}$ for all odd $i\in[2\dep]$ and  $\mscale_i\in\bR^{\wid\times \wid}$, $\weights_i\in\bR^{d^2\times d}$ for all even $i\in[2\dep]$. Then one can compute the \acl{BN} parameters $\mscale_1\dots \mscale_{2l}$ and $\mshift_1\dots \mshift_{2l}$, so that  with probability 1 the two networks are equivalent $f\equiv g$, i.e., $f(\data)=g(\data)$ for all $\data\in\R^d$ with $\|\data\|\leq 1$.

\end{theorem}

\begin{proof}
To prove this theorem, we apply  \cref{lem: approximation of one layer} layer-wise for every layer of $g$, to get corresponding two layers of $f$.  
\end{proof}

\begin{remark}
    The size of the network can be the same throughout all layers, by zero padding the matrices, without changing the result.
\end{remark}

\section{Depth Overparameterization}\label{app: depth}
\begin{figure}
   \centering
   \includegraphics[trim={0 18em 6em 0}, clip, width=\textwidth]{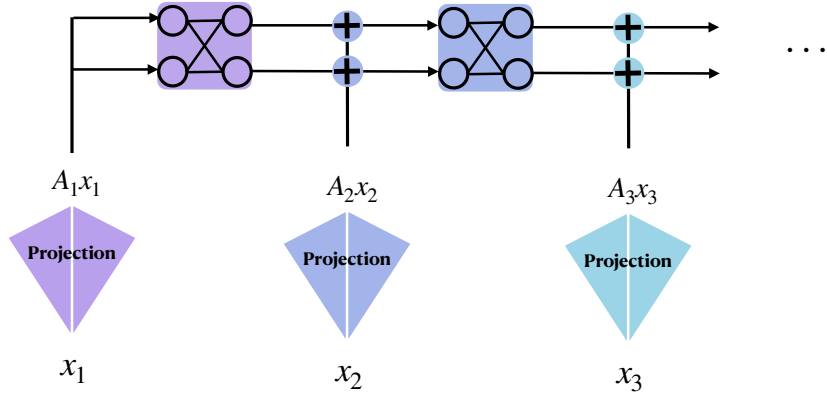}
   \caption{A visualization of the skip connections.  $\data_i$ is the $i$-th partition of the input $\data$.}
   \label{fig: deep network}
\end{figure}

\begin{tcolorbox}[enhanced,width=5.5in, drop fuzzy shadow southwest,    boxrule=0.4pt,,colframe=yellow!50!blue,colback=blue!10]
In this section we will show how the combination of a form of skip connections and depth can lead to the same results, with a reduced width at each layer. In \cref{fig: deep network} we have an illustration of the architecture used. 
\end{tcolorbox}
\begin{lemma}\label{lem: non zero component} Let $\rmA\in\mat^{n\times nm}$, $\rmC\in\mat^{m\times nm}$ be random matrices, $\rmD\in\mat^{nm\times nm}$ a diagonal fixed matrix with non-zero elements and $\iw$ a fixed vector which is  not the zero vector. Define $\rmB = (\rmA\rmD)\kar\rmC$, then
\begin{equation}
    \probof{\inner{\ermB^{-1}_i}{\iw} = 0} = 0
\end{equation}
where $\ermB^{-1}_i$ is the $i-$th row of the inverse of $\rmB$.

\end{lemma}       
\begin{proof}
Without loss of generality, we will focus on the first row of $\rmB^{-1}$, $\ermB^{-1}_1$ and we will drop the index $1$. We know that the elements of this row are $(\ermB^{-1})_{i} = \frac{1}{\det(\rmB)}(-1)^{1+i}\det_{-i,-1}$, where $\det_{-i,-1}$ is the determinant of matrix $\rmB$ with the $i-$th row and the first column deleted. So 
\begin{equation}\label{eq: determinant}
    \inner{\ermB^{-1}}{\iw} = \sum_{i=1}^{nm}(-1)^{1+i}\iw_i \det_{-i,-1}
\end{equation}
where we skipped the term $\frac{1}{\det(\rmB)}$ because it doesn't affect whether the inner product is zero or not.  Notice that \cref{eq: determinant} is actually the determinant of $\rmB$ if we substitute the first column with the vector $\iw$. Our matrix is now
\begin{equation}
    \rmB' =\begin{bmatrix}
       w_1^* & a_{1,2}d_{2} c_{1,2} & \hdots &a_{1,nm}d_{nm} c_{1,nm}\\
       w_2^* & a_{1,2}d_{2}c_{2,2}  & \hdots &a_{1,nm}d_{nm} c_{2,nm}\\
       \vdots& \vdots& \vdots &\vdots\\
        w_m^* & a_{1,2}d_{2}c_{m,2}  & \hdots &a_{1,nm}d_{nm} c_{m,nm}\\
       \vdots& \vdots& \ddots &\vdots\\
       w_{nm-m+1}^* & a_{n,2}d_{2} c_{1,2} & \hdots &a_{n,nm}d_{nm} c_{1,nm}\\
       \vdots& \vdots& \vdots &\vdots\\
       w_{nm}^* & a_{n,2}d_{2} c_{m,2} & \hdots &a_{n,nm}d_{nm} c_{m,nm}
    \end{bmatrix}
\end{equation}
Since $\iw$ is not a zero vector, there exist at least one non zero element of $\iw$. Again without loss of generality, assume that $w_1^*$ is that element. Then, we assign to the matrices $\rmA, \rmC$ the same values as in the proof of  \cref{lem: KR full rank }. In that case, the matrix is now the diagonal matrix $\rmD$ with the first row substituted by the vector $\iw$. This matrix is full rank, since all columns and rows are linearly independent.
\end{proof}

\begin{theorem}[Restatement of \cref{th: depth}]\label{th: depthapp}
Consider the following neural network
\begin{equation*}
    g(\data) = \im[l]\relu(\hdots\relu(\im[1]\data))
\end{equation*}
where $\im[i]\in\mat^{d\times d}$ with $\|\im[i]\|\leq 1$ for all $i=1,\hdots,l$.

Let the  $j(\slayers +1)$-th layer of the neural network $f$ be 
\begin{align*}
    f^j(\altdata) &= \relu(\weights^{j}_{\slayers+1}L^{j}_{\slayers }) \text{ for } j=1,\hdots, l,\\
    \text{where,}&\altdata\in\R^d,\\
    &\weights^j_{\slayers+1}\in\mat^{d\times dk}, \text{ for } j=1,\hdots,l\\
    &L^{j}_1 = \mscale^{j}_1\weights^{j}_1\altdata^{j}_1 + \mshift^{j}_1,\\
    &L^{j}_{i} = \mscale^{j}_i\weights^{j}_i[\relu\parens{L^{j}_{i-1}}+\ProjInput^{j}_i\altdata^{j}_i] + \mshift^{j}_i\;\text{ for }i = 2,\hdots, \slayers\\
    &\weights^j_i,\mscale^j_i\in\mat^{dk\times dk}, \ProjInput_i^j\in\mat^{dk\times k},\text{ and }\mshift^j_i\in\R^{dk}\text{ for all }i=1,\hdots \slayers, j=1,\hdots,l.
\end{align*}
Here the matrices $\weights_i^j, \forall i,j$ are randomly initialized and then frozen. Then,  with probability 1, one can compute the \acl{BNs} parameters  so that the two networks are equivalent, meaning
\begin{equation}
    g(\data) = f^\dep(f^{\dep-1}(\hdots f^1(\data))), \forall \|\data\|\leq 1
\end{equation}
Here, the parameter $k$ is tunable and can be any integer in  $[d]$.
 
\end{theorem}
\begin{remark}
Before we proceed with the proof of this theorem we would like to highlight that for the last layer, we would need not to include the \acl{relu} in the constructed neural network.
\end{remark}
\begin{proof} We start by giving some intuition on how to use the skip connections in order to employ deeper but thinner networks. The idea is to partition the input, pass a different piece of the input in each individual layer and try to exactly reconstruct the parts of the target matrix $\im$ that correspond to the specific piece. Let $\chunk $ be the size of each partition and so in total we need to have $\slayers$ layers. More specifically, notice that
\begin{align}\label{eq: summation}
   \im\data &= \begin{bmatrix}
   \inner{\im[1]}{\data}\\
   \vdots\\
   \inner{\im[\wid]}{\data}
   \end{bmatrix}\\
   &= \begin{bmatrix}
   \sum_{i=1}^\wid\iw_{1,i}\data_i\\
   \vdots\\
   \sum_{i=1}^\wid\iw_{\wid,i}\data_i
   \end{bmatrix}\\
   &= \begin{bmatrix}
   \sum_{i=1}^\chunk\iw_{1,i}\data_i+ \sum_{i=\chunk+1}^{2\chunk}\iw_{1,i}\data_i  +\hdots +\sum_{i=(\slayers-1)\chunk+1}^{\wid}\iw_{1,i}\data_i\\
   \vdots\\
   \sum_{i=1}^\chunk\iw_{\wid,i}\data_i+ \sum_{i=\chunk+1}^{2\chunk}\iw_{\wid,i}\data_i  +\hdots +  \sum_{i=(\slayers-1)\chunk+1}^{\wid}\iw_{\wid,i}\data_i
   \end{bmatrix}\\
   &= \begin{bmatrix}
   \sum_{i=1}^\chunk\iw_{1,i}\data_i \\
   \vdots\\
   \sum_{i=1}^\chunk\iw_{\wid,i}\data_i
   \end{bmatrix}
   +
    \begin{bmatrix}
   \sum_{i=\chunk+1}^{2\chunk}\iw_{1,i}\data_i \\
   \vdots\\
   \sum_{i=\chunk+1}^{2\chunk}\iw_{\wid,i}\data_i
   \end{bmatrix}
   + \hdots 
   + \begin{bmatrix}
    \sum_{i=(\slayers-1)\chunk+1}^{\wid}\iw_{1,i}\data_i\\
   \vdots\\
   \sum_{i=(\slayers-1)\chunk+1}^{\wid}\iw_{\wid,i}\data_i
   \end{bmatrix}\\
   &= \im \Subsel^{T}_{1}\Subsel_1 \data + \im\Subsel_2^T\Subsel_2\data+\hdots+\im\Subsel_{\slayers}^T \Subsel_{\slayers}\data
\end{align}
where $\Subsel_i\in\mat^{k\times d}$ is a sub-selection matrix and specifically
\begin{equation}
    \Subsel_i = \begin{bmatrix}
      \zero &\rvline & \id_{\chunk} &\rvline &\zero
    \end{bmatrix}
\end{equation}
where the identity is in the place of $[(i-1)\chunk +1, i\chunk]$ columns. 

Having this in mind, we will use each one of the layers with the skip connections to approximate a specific part of the input, which
is then projected to a $d\chunk$-dimensional vector  through the matrix $\ProjInput_i$. 
\begin{equation}
    f(\data) = \weights_{\slayers+1}L_{\slayers }
\end{equation}
where $L_1 = \mscale_1\weights_1\data_1 + \mshift_1$ and
$L_{i} = \mscale_i\weights_i[\relu\parens{L_{i-1}}+\ProjInput_i\data_i] + \mshift_i$ for $i = 2,\hdots, \slayers-1$. Before we proceed further , let us note the dimensions of each matrix and vector. 
1) $\data_i = \Subsel_i\data\in\R^{\chunk}$ 
2)~$\ProjInput_i\in\mat^{d\chunk\times\chunk}$ is the matrix that projects the input to a higher dimension and it is randomly initialized
3) $\weights_i\in\mat^{d\chunk\times d\chunk}$ randomly initialized 
4) $\mscale_i\in\mat^{d\chunk\times d\chunk}$, $\mshift_i\in\R^{d\chunk}$ the parameters that we control for all $i=1,\hdots,\slayers$. The last linear layer is added to correct the dimensions and project everything back to $\R^d$; thus $\weights_{\slayers+1}\in\mat^{d\times d\chunk}$.

\paragraph{Linearization. } We will set the parameters $\mshift_i$, $i=1,\hdots,\slayers-1$ to be such that all \acl{relu}s are activated as we did in the proof of \cref{th: width app}, while by controlling $\mshift_{\slayers}$ we will cancel out any error created.

\paragraph{Segmentation. }After the previous step has been completed, the output of the network becomes 
\begin{equation}\label{eq: deep net}
    f(\data) = \weights_{\slayers+1}\parens*{\prod_{i=1}^{\slayers}\mscale_i\weights_i\ProjInput_1\data_1+ \prod_{i=2}^{\slayers}\mscale_i\weights_i\ProjInput_2\data_2+\hdots +\mscale_{\slayers}\weights_{\slayers}\ProjInput_{\slayers}\data_{\slayers}}
\end{equation}
We will now use each one of the terms in the summation to reconstruct each one of the terms in the last part of \cref{eq: summation}. Thus we will try to solve the  system of linear systems. 
We will show that this is feasible, by using  induction.

\paragraph{Base step.} We start with the last equation which is the simplest and we have
\begin{equation}
    \weights_{\slayers+1}\mscale_{\slayers}\weights_{\slayers}\ProjInput_{\slayers}\data_{\slayers} = {\im}_{\slayers}\data_{\slayers}
\end{equation}
It is sufficient to show that
\begin{equation}
    \weights_{\slayers+1}\mscale_{\slayers}\weights_{\slayers}\ProjInput_{\slayers} = \im[\slayers]
\end{equation}
where $\im[\slayers] = \im\Subsel_{\slayers}^T$. From \cref{lem: solution with khatri-rao} we know that the above system can be rewritten as
\begin{equation}
    \weights_{\slayers+1}\kar(\ProjInput_{\slayers}^T\weights_{\slayers}^T)\mathrm{vecd}(\mscale_{\slayers}) = \mathrm{vec}(\im[\slayers])
\end{equation}
which from \cref{rem: KR more matrices} is solvable with probability $1$. By \cref{lem: non zero component} we also have that  with probability $1$ all the elements of $\mscale_{\slayers}$ are non-zero.

\paragraph{Inductive step. }Assume now that all components up to (and not included) the $i-th$ component have been exactly reconstructed and $\mscale_j$, $j=\slayers,\hdots, i-1$ have non-zero elements with probability one.

\paragraph{Last step.} We will now show that the $i-th$ component of  \cref{eq: deep net} can exactly reconstruct the $i-th$ component of \cref{eq: summation}. We want 
\begin{equation}
    \weights_{\slayers+1}\prod_{j=i}^{\slayers}\mscale_j\weights_j\ProjInput_i\data_i = \im[i]\data_i
\end{equation}
Again it is sufficient to show that 
\begin{equation}
    \weights_{\slayers+1}\prod_{j=i}^{\slayers}\mscale_j\weights_j\ProjInput_i = \im[i]
\end{equation}
We can rewrite this system using \cref{lem: solution with khatri-rao} with respect to $\mscale_i$ as
\begin{equation}
   \bracks*{ (\weights_{\slayers+1}\prod_{j=i-1}^{\slayers}\mscale_j\weights_j)\kar(\ProjInput_i^T\weights_i^T)}\vecd(\mscale_i) =\myvec(\im[i])
\end{equation}
We now need to show that the first matrix is full rank with probability one. 
From \cref{th: mixed products} part 1 we can rewrite this matrix as
\begin{align}
    (\weights_{\slayers+1}\prod_{j=i-1}^{\slayers}\mscale_j\weights_j)&\kar(\ProjInput_i^T\weights_i^T) = \parens{(\weights_{\slayers+1}\prod_{j=i-2}^{\slayers}\mscale_j\weights_j)\kron\ProjInput_i^T}\parens*{(\mscale_{i-1}\weights_{i-1})\kar\weights_i^T}\\
    &= (\weights_{\slayers+1}\kron\ProjInput_i^T)\prod_{j=i-2}^{\slayers}(\mscale_j\kron \id_{d\chunk})(\weights_j\kron\id_{d\chunk})(\mscale_{i-1}\kron\id_{d\chunk})(\weights_{i-1}\kar\weights_i^T)
\end{align}
where we applied for the last equation \cref{th: mixed products} part 1 one more time and \cref{th: kronecker} part 2 to split the Kronecker product to product of Kronecker products. We now want to argue about the rank of the matrix above. Notice that from \cref{th: kronecker} part 1, all the intermediate Kronecker products result in matrices of rank $d^2\chunk^2$ which means that they are all full rank and thus they have an empty nullspace. Furthermore, from \cref{fact: multiplication is full rank} we have that 
\begin{equation}
    \mathrm{Null}((\weights_{\slayers+1}\kron\ProjInput_i^T)\prod_{j=i-2}^{\slayers}(\mscale_j\kron \id_{d\chunk})(\weights_j\kron\id_{d\chunk})(\mscale_{i-1}\kron\id_{d\chunk})) = \mathrm{Null}((\weights_{\slayers+1}\kron\ProjInput_i^T))
\end{equation}
Thus, it is sufficient to show that the matrix $(\weights_{\slayers+1}\kron\ProjInput_i^T)(\weights_{i-1}\kar\weights_i^T)$ is full rank, meaning that the null space of the Kronecker product has no intersection with the column space of the Khatri-Rao product. Also, using again \cref{th: mixed products} we have that 
\begin{equation}
    (\weights_{\slayers+1}\kron\ProjInput_i^T)(\weights_{i-1}\kar\weights_i^T) = (\weights_{\slayers+1}\weights_{i-1})\kar(\ProjInput_i^T\weights_i^T)
\end{equation}
By using \cref{rem: KR more matrices} we conclude that the matrix is full rank. To also prove that $\mscale_i$ does not have any zero elements, we need to find an assignment of the elements of the  matrices such that  any $\myvec(\im[i])$ is not orthogonal to any of the rows of the inverse of the matrix $(\weights_{\slayers+1}\prod_{j=i-1}^{\slayers}\mscale_j\weights_j)\kar(\ProjInput_i^T\weights_i^T) $ with probability $1$. 

Notice that each of the elements of the inverse of a matrix $\rmB$ is $B^{-1}_{ij} = \frac{1}{\det(\rmB)}(-1)^{i+j}\det_{-j,-i}$, where $\det_{-j,-i}$ is the determinant of the matrix $B$ if we delete the $j-$th row and the $i-$th column.
Without loss of generality we will just pick the first row of the inverse matrix and prove the result. Let 
$$\rmB = \parens{(\weights_{\slayers+1}\prod_{j=i-1}^{\slayers}\mscale_j\weights_j)\kar(\ProjInput_i^T\weights_i^T)}^{-1}$$
The first row of this matrix is $\rmB_1 = [(-1)^2\det_{-1,-1} , \hdots, (-1)^{1+i}\det_{-i,-1}, \hdots, (-1)^{1+dk}\det_{-dk,-1}]$. We want to show that $\probof{\inner{\rmB_1}{\myvec(\im[i])} = 0} = 0$ for any fixed $\myvec(\im[i])$ that it's not of course identical to zero. From \cref{lem: measure zero} it is sufficient to find an assignment of the values of the matrices that gives non-zero result always. Let us choose all $\weights_j = \id_{dk}$ for all $j=i,\hdots, dk$. So,  it is sufficient to find an assignment for the rest such that the first row of  $\parens{(\weights_{\slayers+1}\mscale)\kar\ProjInput_{i}^T}^{-1}$ is not orthogonal to a fixed vector $\myvec(\im[i])$, where $\mscale = \prod_{j=i-1}^{\slayers}\mscale_j $ and it is a diagonal matrix with non-zero elements. By using \cref{lem: non zero component} we conclude the proof.
\end{proof}
\begin{remark}
Except for the total number of parameters, which we have already mentioned, we would like to pinpoint that the architecture in this case  is slightly different from a fully connected neural network, since it had an extra linear layer, before the activation function.
\end{remark}

\section{Sparse Matrix Inverse}\label{app: sparse}

\begin{tcolorbox}[enhanced,width=5.5in, drop fuzzy shadow southwest,    boxrule=0.4pt,,colframe=yellow!50!blue,colback=blue!10]
We now show that the random weight matrices of \cref{th: width app} can be randomly sparsified and retain the invertibility of the Khatri-Rao product with high probability.  The intuition of this result lies in the fact that the invertibility of a sparsified matrix is mainly based on the sparsity patterns that are created. 
\end{tcolorbox}

To do so, we first define the notion of Boolean determinant\footnote{ For ease of notation we index the columns and rows of matrices starting from $0$.}
\begin{definition}
The Boolean determinant of a $d\times d$ Boolean matrix $\rmB$ is recursively defined to be
\begin{itemize}
\item if $d=1$, $\bdet (\rmB):=\ermB_{0,0}$, that is, it is the only element of the matrix, and
    \item if $d>1$,
    \begin{align*}
    \bdet (\rmB):=\max_{i\in \{0,\dots, d\}}(\ermB_{i,0}\bdet(\rmB_{-(i,0)})),
\end{align*}
where $\ermB_{i,j}$ is the element of $\rmB$ at $i$-th row and $j$-th column, and $\rmB_{-(i,j)}$ is the sub-matrix of $\rmB$ with $i$-th row and $j$-th column removed. Note that $\bdet(\rmB)\in \{0,1\}$ for any Boolean matrix $\rmB$. 
\end{itemize}
\end{definition}
In what follows, we will prove some auxiliary lemmas in order to prove \cref{th: sparse}.

\begin{lemma}\label{lem:Boolean nonzero}
Let $\rmA$,$\rmB$ be two $d^2\times d^2$ Bernoulli matrices and $\rmP\in\mat^{d\times d^2}$, $\rmQ\in\mat^{d\times d^2}$ be two random matrices. Then  $(\rmP\had\rmA) \kr (\rmQ\had\rmB)$ is invertible if and only if  $\bdet(\rmA\kr \rmB)=1$. 
\end{lemma}
\begin{proof} Let $\rmU = \rmP\had\rmA$ and $\rmV = \rmQ\had\rmB$. 
We will prove the statement by induction on square submatrices of $(\rmP\had\rmA) \kr (\rmQ\had\rmB)$ and the corresponding square submatrices of  $\rmA\kr \rmB$, that is, we prove that $\text{Det}(((\rmP\had\rmA) \kr (\rmQ\had\rmB))_{\cI,\cJ})\neq 0$ if and only if  $\bdet((\rmA\kr \rmB)_{\cI,\cJ})=1$, where $\cI\subset [d^2]$ is the set of selected rows and $\cJ\subset [d^2]$ is the set of selected columns.

For size 1 submatrices, that is, when $|\cI|=1$ and $|\cJ|=1$, the submatrix of $(\rmP\had\rmA) \kr (\rmQ\had\rmB)$ will contain just one element, which will be of the form $pqab$, where $p$, $q$, $a$ and $b$ are the elements of the $\rmP$, $\rmQ$, $\rmA$ and $\rmB$ respectively. The corresponding submatrix of $\rmA\kr \rmB$ will just contain $ab$, and hence with probability 1, $pqab\neq 0 \iff ab = 1$. This proves the base case.

Now, we let $\rmU = \rmP\had\rmA $ and $\rmV = \rmQ\had\rmB$ and we move on to the case for $|\cI|=|\cJ|>1$. Here, we want to show that with probability $1$,  $$\text{Det}((\rmU \kr \rmV)_{\cI,\cJ})= 0 \iff \bdet((\rmA \kr \rmB)_{\cI,\cJ})= 0$$

We write the expression for determinant of $(\rmU \kr \rmV)_{\cI,\cJ}$ calculated along its first column. Without loss of generality, assume that $0\in \cJ$. Then, 
\begin{align*}
    \text{Det}((\rmU \kr \rmV)_{\cI,\cJ}) = \sum_{i\in \cI} \ermA_{\lfloor i/d \rfloor,0}\ermB_{i \Mod{d},0}\ermP_{%
    \lfloor i/d \rfloor,0 }\ermQ_{i\Mod{d},0} \text{Det}((\rmU \kr \rmV)_{\cI\setminus \{i\},\cJ\setminus\{0\}}),
\end{align*}
We can rewrite the summation above as 
\begin{align*}
    \text{Det}((\rmU \kr \rmV)_{\cI,\cJ}) = \sum_{j=0}^{d-1} \ermA_{j, 0}\ermP_{j, 0}\sum_{l\in \{l: jd+l\in \cI\}}\ermB_{l, 0}\ermQ_{l, 0} \text{Det}((\rmU \kr \rmV)_{\cI\setminus \{jd+l\},\cJ\setminus\{0\}}).
\end{align*}
Note that $\ermA_{j, 0}\ermP_{j, 0}$ is independent of $\sum_{l\in \{l: jd+l\in \cI\}}\ermB_{l, 0}\ermQ_{l, 0} \text{Det}((\rmU \kr \rmV)_{\cI\setminus \{jd+l\},\cJ\setminus\{0\}})$, since $\rmA$, $\rmP$, $\rmB$, and $\rmQ$ are all independent of each other, and the columns of $\rmU \kr \rmV$ are also independent. Since $\rmP$ and $\rmQ$ have elements coming from continuous density distribution, with probability 1, the above expression is $0$ if and only if each term in the sum above is individually $0$.\footnote{If some of the term is not zero, then this is a polynomial with respect to the variables $(\ermP,\ermQ)$ and it is zero in a set of measure zero.}. That is, the expression above is 0 if and only if all the terms of the form $\text{Det}((\rmU \kr \rmV)_{\cI\setminus \{jd+l\},\cJ\setminus\{0\}})$ are 0. Using the inductive hypothesis we get that $$\text{Det}((\rmU \kr \rmV)_{\cI\setminus \{jd+l\},\cJ\setminus\{0\}})= 0 \iff \bdet((\rmA \kr \rmB)_{\cI\setminus \{jd+l\},\cJ\setminus\{0\}})= 0$$. Hence  with probability 1, $\text{Det}((\rmU \kr \rmV)_{\cI,\cJ}) = 0$ if and only if 
\begin{align*}
    \sum_{j=0}^{d-1} \ermA_{j, 0}\sum_{l\in \{l: jd+l\in \cI\}}\ermB_{l, 0}\bdet((\rmA \kr \rmB)_{\cI\setminus \{jd+l\},\cJ\setminus\{0\}})=0.
\end{align*}
The above expression is $0$ if and only if $$\max_{j\in\{0,\dots,d-1\},l\in \{l: jd+l\in \cI\}} (\ermA_{j, 0}\ermB_{l, 0}\bdet((\rmA \kr \rmB)_{\cI\setminus \{jd+l\},\cJ\setminus\{0\}}))=0.$$ 
Noting that $$\max_{j\in\{0,\dots,d-1\},l\in \{l: jd+l\in \cI\}} (\ermA_{j, 0}\ermB_{l, 0}\bdet((\rmA \kr \rmB)_{\cI\setminus \{jd+l\},\cJ\setminus\{0\}})) =  \bdet((\ermA\kr\ermB)_{\cI,\cJ})$$ concludes the proof for the induction step.
\end{proof}

\begin{lemma}\label{lem:substitute col}
Let $\ru_0,\dots,\ru_{d-1}$ and $\rv_0,\dots,\rv_{d-1}$ be Bernoulli random variables with probability $p$ of success; and let $\rx_0,\dots,\rx_{d^2-1}$ be Bernoulli random variables with probability $q$ of success, all sampled i.i.d. , with $p\geq \sqrt{2qd}$. Then, for any subset $\cS\subset \{0,\dots, d^2-1\}$,
\begin{align*}
    \bP\left(\sum_{i\in \cS}\ru_{\lfloor i/d \rfloor}\rv_{i\Mod{d}}=0\right)\leq \bP\left(\sum_{i\in \cS}\rx_{i}=0\right)
\end{align*}
\end{lemma}
\begin{proof}
Consider any $\cS$ and define the set $\cA_l = \braces{j: j= i\Mod{d} \land  l=\lfloor i/d \rfloor \text{ for some }i\in\cS}$ for $l =0,\hdots,d-1$. Then,
\begin{align}
    \sum_{i\in \cS}\ru_{\lfloor i/d \rfloor}\rv_{i\Mod{d}}&=\sum_{l=0}^{d-1}\ru_l\sum_{j\in \cA_l}\rv_{j}
\end{align}
Hence,
\begin{align}
\bP\left(\sum_{i\in \cS}\ru_{\lfloor i/d \rfloor}\rv_{i\Mod{d}}\right)&=\bP\left(\sum_{l=0}^{d-1}\ru_l\sum_{j\in \cA_l}\rv_{j}=0\right)\\
    &=\bP\left(\sum_{l=1}^{d-1}\ru_l\sum_{j\in \cA_l}\rv_{j}=0\right)\bP\left(\ru_0\sum_{j\in \cA_0}\rv_{j}=0\middle| \sum_{l=1}^{d-1}\ru_l\sum_{j\in \cA_l}\rv_{j}=0\right).
\end{align}
Looking at the term $\bP\left(\ru_0\sum_{j\in \cA_0}\rv_{j}=0\middle| \sum_{l=1}^{d-1}\ru_l\sum_{j\in \cA_l}\rv_{j}=0\right)$, note that if $|\cA_0\cap \bigcup_{l=1}^{d-1}\cA_l|=0$, then $\bP\left(\ru_0\sum_{j\in \cA_0}\rv_{j}=0\middle| \sum_{l=1}^{d-1}\ru_l\sum_{j\in \cA_l}\rv_{j}=0\right)=\bP\left(\ru_0\sum_{j\in \cA_0}\rv_{j}=0\right)$. 
Otherwise, $$\bP\left(\ru_0\sum_{j\in \cA_0}\rv_{j}=0\middle| \sum_{l=1}^{d-1}\ru_l\sum_{j\in \cA_l}\rv_{j}=0\right)\geq\bP\left(\ru_0\sum_{j\in \cA_0}\rv_{j}=0\right)$$. 
In order to show this, assume without loss of generality, that $|\cA_0\cap \bigcup_{l=1}^{d-1}\cA_l|=k$. This means that at least $k$ elements that are contained in $\cA_0$ is also contained  in at least one of the other $\cA_l$ sets, this immediately implies  that at least $k$ terms in the summation of $\sum_{j\in\cA_0}\rv_j $ are zero. Notice that 
\begin{align}
\probof{\ru_0 \sum_{j\in\cA_0}\rv_j  =0 }
 &= \probof{\ru_0 = 0} + \probof{\sum_{j\in\cA_0}\rv_j  =0} - \probof{\braces{\ru_0=0}\cap\braces{\sum_{j\in\cA_0}\rv_j  =0}}\\
 &= \probof{\ru_0 = 0} + \probof{\sum_{j\in\cA_0}\rv_j  =0} - \probof{\braces{\ru_0=0}}\probof{\braces{\sum_{j\in\cA_0}\rv_j  =0}}
 \end{align}
 Since $\ru_0$ is independent of $\rv_j$s. It is obvious that if we know that at least one of the $\rv_j$s is zero, then the probability that the sum of them is zero is increased ($\rv_j$s are also independent to each other). Thus the more $\rv_j$s we know that are zero the more this function is increased. (We can write the last expression as $p+q^{\abs{\cA_0} -k} - pq^{\abs{\cA_0} -k}$ and it is easy to see that as $k$ is increased the value of the expression is increased).
 
Thus, to maximize $\bP\left(\sum_{i\in \cS}\ru_{\lfloor i/d \rfloor}\rv_{i\Mod{d}}\right)$, we need to maximize the size of overlaps of the form $\cA_0\cap \bigcup_{l=1}^{d-1}\cA_l$. To proceed now, let $a_l:=|\cA_l|$. Then,
\begin{align}
\bP\left(\sum_{i\in \cS}\ru_{\lfloor i/d \rfloor}\rv_{i\Mod{d}}\right)&=\bP\left(\sum_{l=0}^{d-1}\ru_l\sum_{j\in \cA_l}\rv_{j}=0\right)\\
&\leq \bP\left(\sum_{l=0}^{d-1}\ru_l\sum_{j=0}^{a_l-1}\rv_{j}=0\right)\tag*{(This maximizes the overlap amongst $\cA_l$'s)}.
\end{align}
Without loss of generality, assume that $a_0\geq a_1\geq \dots$. Then, let $h$ be the smallest integer such that $|\{i:a_i \geq h \}| \leq h $. Next, we claim that $\abs{\cS}=: n \leq 2dh$. To see how, let $\cT_1=\{i:a_i\geq h\}$ and $\cT_2=\{i:a_i < h\}$. Then, $n=\sum a_i=\sum_{i\in \cT_1}a_i+\sum_{i\in \cT_2}a_i\leq dh+dh=2dh$, since $|\cT_1|\leq h$, $a_i\leq h$ for $a_i$ in $\cT_2$, $a_i\leq d$ and $|\cT_2|\leq d$.

Continuing on from the equation above, 
\begin{align*}
\bP\left(\sum_{i\in \cS}\ru_{\lfloor i/d \rfloor}\rv_{i\Mod{d}}\right)&\leq \bP\left(\sum_{l=0}^{d-1}\ru_l\sum_{j=0}^{a_l-1}\rv_{j}=0\right)\\
&\leq \bP\left(\sum_{l=0}^{h-1}\ru_l\sum_{j=0}^{h-1}\rv_{j}=0\right).
\end{align*}
The sum above is $0$ if and only if $\sum_{l=0}^{h-1}\ru_l=0$ or $\sum_{j=0}^{h-1}\rv_{j}=0$. The probability of that is $2(1-p)^h - (1-p)^{2h} =(1-p)^h(2-(1-p)^h)$. On the other hand, $\bP\left(\sum_{i\in \cS}\rx_{i}=0\right)=(1-q)^n\geq (1-q)^{2dh}$. Hence, we need the following to be true:
\begin{align*}
    &(1-p)^h(2-(1-p)^h)\leq (1-q)^{2dh}\\
    \text{or}\quad &(1-p)(2-(1-p)^h)^{1/h}\leq (1-q)^{2d}
\end{align*}
Note that $(2-(1-p)^h)^{1/h}\leq 1+p$ and $(1-q)^{2d}\geq 1-2qd$. Hence, it is sufficient for the following to hold:
\begin{align*}
    (1-p)(1+p)\leq 1-2qd.
\end{align*}From the above, we get that $p\geq \sqrt{2qd}$ suffices.
\end{proof}

This lemma shows that if we replace a column of $\rmA\kr \rmB$ with i.i.d.  Bernoulli random variables with probability of success $q$, then the probability that the new matrix has Boolean determinant zero is higher than that of the original one. To see how, just let $\cS$ be the set of indices $i$ for which $\bdet((\rmA\kr \rmB)_{-(i,0)}) = 1$. Then, the LHS in the statement of this lemma is just the probability of $\bdet(\rmA\kr \rmB) = 0$ and RHS is the probability of the new matrix having Boolean determinant zero. Since the columns of $\rmA\kr \rmB$ are independent, we can keep replacing the columns one-by-one to get a matrix with all elements sampled i.i.d. from ${Bern}(q)$. Then, we can use the following lemma on the new matrix.

We now restate a theorem from \cite{Basak2018}. Though mentioned in the main part, we also state it here for completeness and clarity:
\begin{theorem}[Theorem 1.1 in \citep{Basak2018}]\label{th: app- Bernoulli}
let $\rmA\in\mat^{d\times d}$ be an a matrix with i.i.d. $Ber(p)$ entries.  Then,  there exist absolute constants $0 < c, \bar{c}, C <\infty$ such that for any $\eps>0$, and $\rho$ such
that $n\rho \geq \log(1/\rho)$, we have
\begin{equation*}
    \prob\parens*{\braces*{(s_{min}(\rmA)\leq \bar{c}\eps\exp(-C\dfrac{\log(1/\rho)}{\log(n\rho)}\sqrt{\dfrac{\rho}{n}}))\cap\Omega_o^c}}\leq \eps + n^{-{c}}
\end{equation*}
where $s_{min}$ is the minimum singular value of the matrix $\rmA$ and $\Omega_0$ is the event that any column or some row is identical to $0$. 
\end{theorem}
\begin{remark}
It is easy to see that $n\rho\geq \log(1/\rho)$ is equivalent with  $\rho\geq \log n/n - \log\log n$.
we note here that we focus on  the invertibility of matrix $\rmA$, that is the probability that $s_{min}(\rmA)>0$. As a result,  let $\eps$ to be exponentially - actually arbitrarily - small  and $\rho = \log n/n$; then the matrix $\rmA$ is non-singular with probability of the order $1/n^c$. We emphasize that the constant $c$ is universal.
\end{remark}

\begin{theorem}[Restatement of \cref{th: sparse}]\label{th: appsparse} For the setting of \cref{th: width app}, consider that  the random matrices of each of the layers are sparsified with probability $p=\Theta(\sqrt{\log{\wid}/\wid})$, meaning that each of their elements is zero with probability $1-p$. If the depth of the network is polynomial in the input, \ie $\dep = \text{poly}(\wid)$ then the results of \cref{th: width app} hold with probability at least $1-1/d$. 
\end{theorem}
\begin{proof}
It is sufficient to prove that after sparsifying the matrices of the two first  layers of our construction, \ie $\weights_1$, $\weights_2$ with $p\geq \Theta(\sqrt{\log d/d})$, the probability that their Khatri-Rao product is no longer invertible is polynomially small in $d$. Let $\rmM_1,\rmM_2$ be \emph{Bern(p)} then from  \cref{lem:Boolean nonzero}
we have that it is sufficient that $\bdet(\rmM_1\kar\rmM_2)\neq 0$. From \cref{lem:substitute col}, we actually get that by substituting one column that has variables which are ${Bern(q)}$, with $p\geq\sqrt{2qd}$ results in matrix which has a higher probability of its Boolean determinant being zero. By applying this argument at each column we get a new matrix $\rmM$, which elements are i.i.d. ${Bern(q)}$. 

We will now compute the value of $q$.
Using \cref{th: Bernoulli}, we get that every Bernoulli i.i.d. matrix $Bern(\rho)$, with $\rho= \log d^2/d^2$, is singular with probability at most $d^{-2c}$ for a universal constant $c$. 
However, we need the probability of being singular to be of the order $d^{-\bar{C}}$ for some $\bar{C}$ which we will determine later in the proof. 
To achieve this smaller probability, consider drawing independently $\rmE_i$, $i\in[\frac{\bar{C}}{c}]$ $Bern(\rho)$ matrices. Then the probability of all of them being singular is at most $d^{-2\bar{C}}$. It is straightforward to see that whenever the determinant of a Boolean matrix is non-zero then the Boolean Determinant of that matrix is also non-zero. Hence, with probability at least $1-d^{-2\bar{C}}$, at least one of these matrices has non-zero Boolean determinant. We create our matrix $\rmM' = \rmE_1 \lor\hdots\lor \rmE_\frac{\bar{C}}{c}$, where the `or' operation is done element-wise. Note that $\rmM' $ has a non-zero Boolean determinant if any of the matrices $ \rmE_1\dots \rmE_\frac{\bar{C}}{c}$ has a  non-zero one. Thus $\rmM'$ has a non-zero Boolean determinant with probability at least $1-d^{-2\bar{C}}$.
Note that $\rmM'$ is itself an i.i.d. Bernoulli matrix, with each element being $Bern(q')$ with $q' \leq  \rho \bar{C} = \bar{C}\log d^2/d^2$. Hence we can set $q= \bar{C}\log d^2/d^2$.

Since $\rmM'$ has a zero Boolean determinant with probability at most  $d^{-2\bar{C}}$, the probability that 
$\bdet(\rmM_1\kar\rmM_2) = 0$ is also smaller than $d^{-2\bar{C}}$, and hence $\rmW_1\kar\rmW_2$ are invertible with probability at least $1-d^{-2\bar{C}}$.
Then to ensure that the Khatri-Rao product of the sparse matrices is invertible for all the $2l$ layers of $f$, we need to do a union bound over all the layers. Thus, the probability that any of the Khatri-Rao products are not invertible is at most $d^{-2\bar{C}}2l$. Since we have assumed that $l=\cO(\text{poly}(d))$, then for a large enough constant $\bar{C}$, the overall probability of the union bound can be driven down to be at most $1/d$. Recalling that $p=\sqrt{2dq}$, we get that $p=\Theta(\sqrt{\log d/d})$ suffices for the theorem to hold.

 \end{proof}

 \section{Experiments}\label{app: experiments}

\begin{figure}[h]
    \centering
    \begin{subfigure}{\textwidth}
         \centering
         \includegraphics[width=0.49\textwidth]{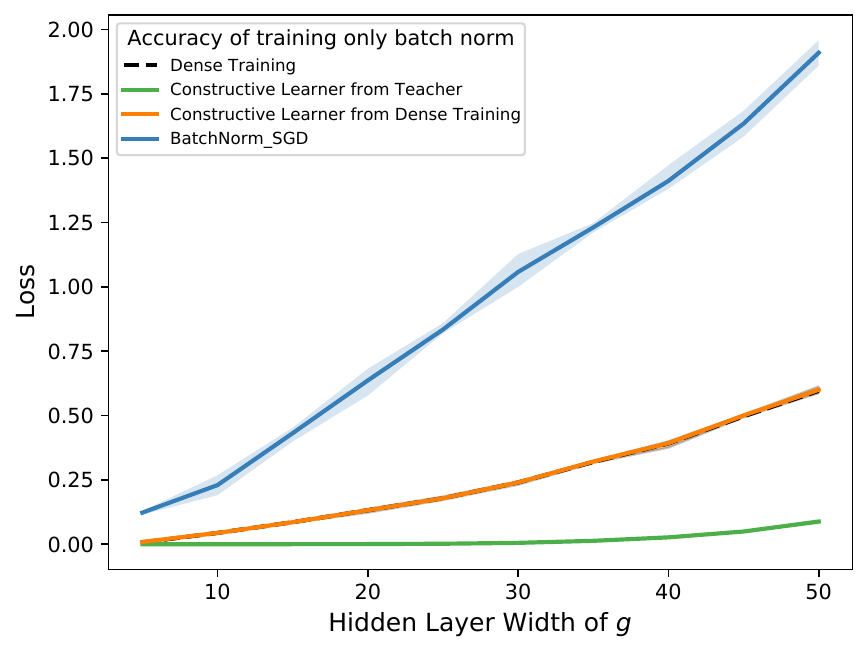}
        \hfill
         \includegraphics[width=0.49\textwidth]{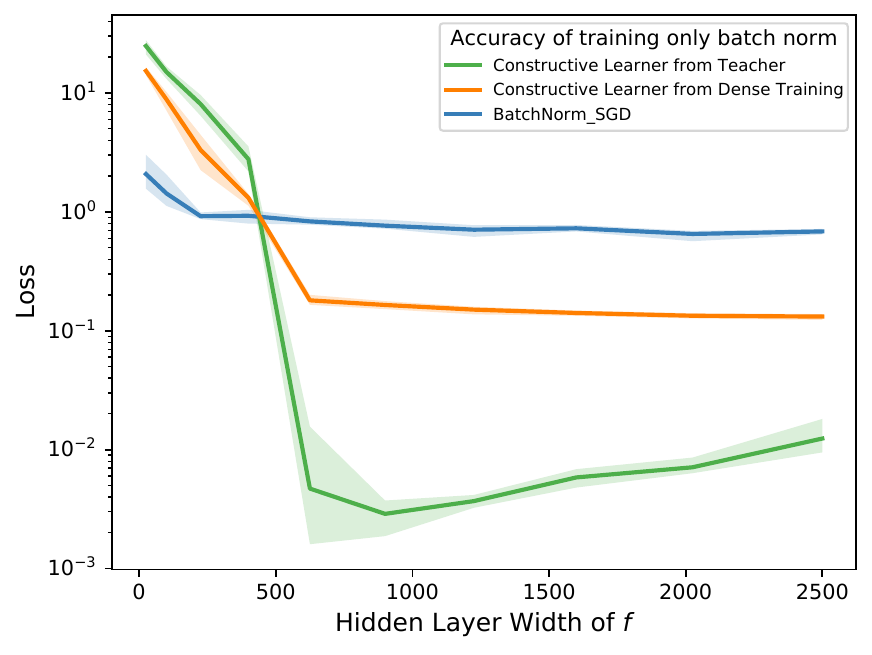}
     \end{subfigure}
     
    \caption{Verifying our constructions for \acl{BNs} parameters on data created by a random neural network $g$. The baseline algorithm (Blue curve) corresponds training the \acl{BNs} parameters of the network $f_1$ on this data using SGD, while the other parameters (weight matrices) are kept frozen at initialization. The green curve represents the network $f_2$ created by our construction by directly using the teacher network $g$. We also train a dense neural network $g'$ with same architecture as $g$ using SGD (dashed black curve). We then use our construction on $g'$ to create $f_3$, which is represented by the orange curve. In the figure on the left, we vary the width of the hidden layer of $g$, and set the width of $f_1,f_2,$ and $f_3$ to be the one proposed in Theorem \ref{thm:width}. In the figure on the right, we fix the width of $g$ to be 25, and vary the widths of $f_1,f_2,$ and $f_3$. }

    \label{fig:exp1}
\end{figure}
To verify our construction numerically, we designed the following experiment: We create a random target network $g(\vx)$ with one hidden layer of width $d$, with ReLU activations. The output of the target network is simply the sum of outputs of the hidden layer neurons. We try to learn this function using the network $f(\vx)$ from Theorem \ref{thm:width} employing three algorithms:
\begin{itemize}
    \item Baseline: We freeze the weight parameters of $f$ and train only the \acl{BNs} parameters using SGD.
    \item Constructive algorithm, directly from target network: We freeze the weight parameters of $f$ and set the \acl{BNs} parameters according to the construction prescribed in the proof of Theorem \ref{thm:width}.
    \item Constructive algorithm, from a dense learnt network: For this approach, we assume that we do not have access to the network $g$ but only to the input/output pairs produced by it. Thus, we first take another randomly initialized network $g'$ with the same architecture as $g$, and train it using SGD on the data generated by $g$. Then, we create $f$ by freezing its weight parameters and setting the \acl{BNs} parameters according to the construction prescribed in the proof of Theorem \ref{thm:width}, but using $g'$ in the construction instead of $g$.  
\end{itemize}
The results are shown in Figure \ref{fig:exp1} (left). We see that both the constructive algorithms beat the baseline. These show that training \acl{BNs} parameters using SGD might not be the best algorithm. While the constructive algorithm which uses the target network directly might not be viable to use in practice (since the parameters of $g$ are not available), the constructive algorithm that first trains a dense network and uses that to construct the \acl{BNs} parameters shows that there can be other, indirect ways of training \acl{BNs} parameters, which are better than simply using SGD. 

To see the dependence of error on the width of $f$ for the three algorithms, we fixed the width of $g$ to be $d=25$, and repeated the first experiment with varying widths of $f$. The results for this experiment are shown in Figure \ref{fig:exp1} (right). Here we see that while at smaller widths the baseline (SGD) performs better (still only achieves a constant loss), as the width increases, our constructive algorithms beat SGD.
Note that here we used pseudo-inverse for setting the \acl{BNs} parameters of $f$ using our construction, since at lower widths, the matrices that need to be inverted become rank-deficient.

\begin{figure}[h]
    \centering
      
    \begin{subfigure}{\textwidth}
         \centering
         \includegraphics[width=0.49\textwidth]{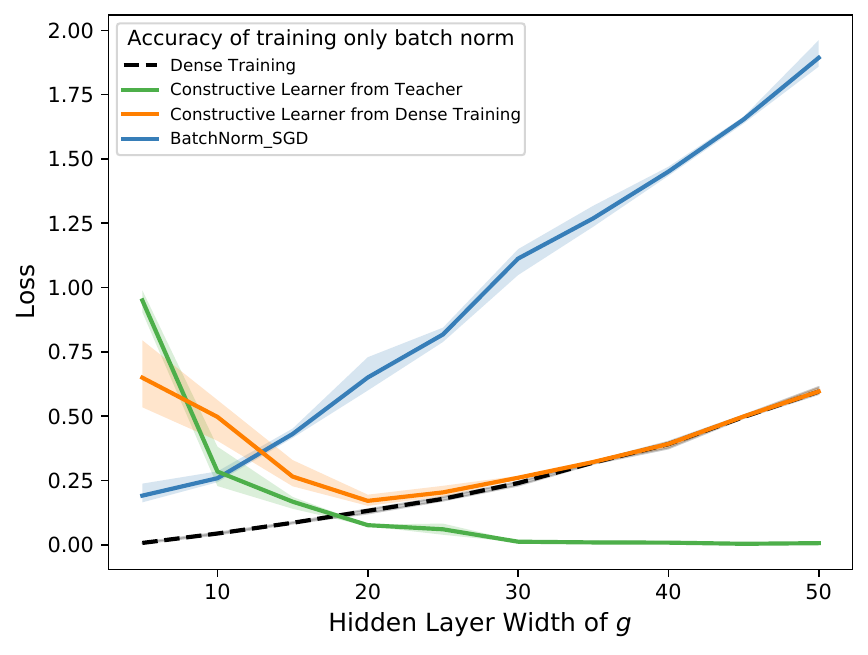}
        \hfill
         \includegraphics[width=0.49\textwidth]{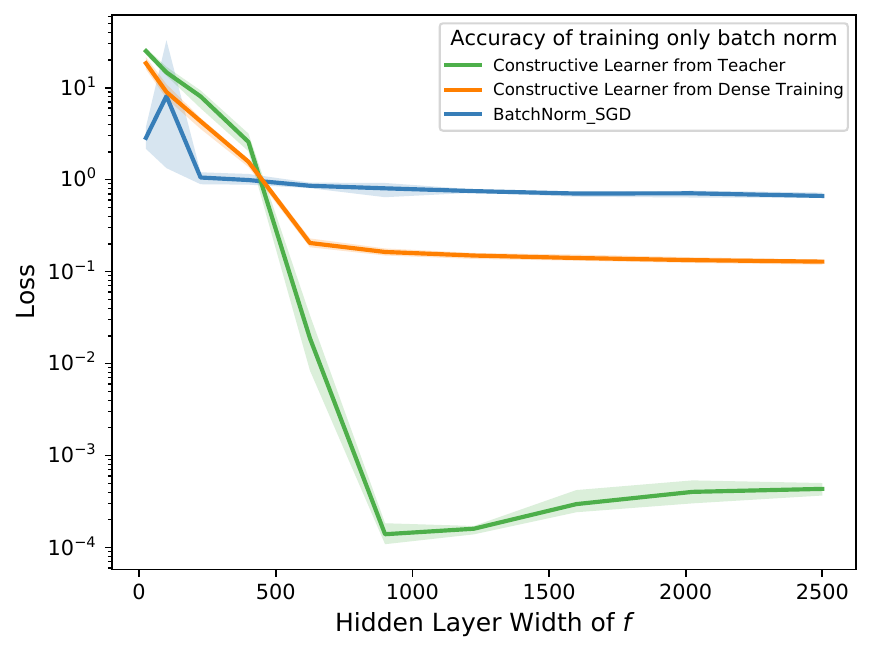}
         \caption{Sparsity = 0.2 (Weight matries have 20\% non-zero weights)}\label{fig:exp2a}
     \end{subfigure}
    \begin{subfigure}{\textwidth} 
         \centering
         \includegraphics[width=0.49\textwidth]{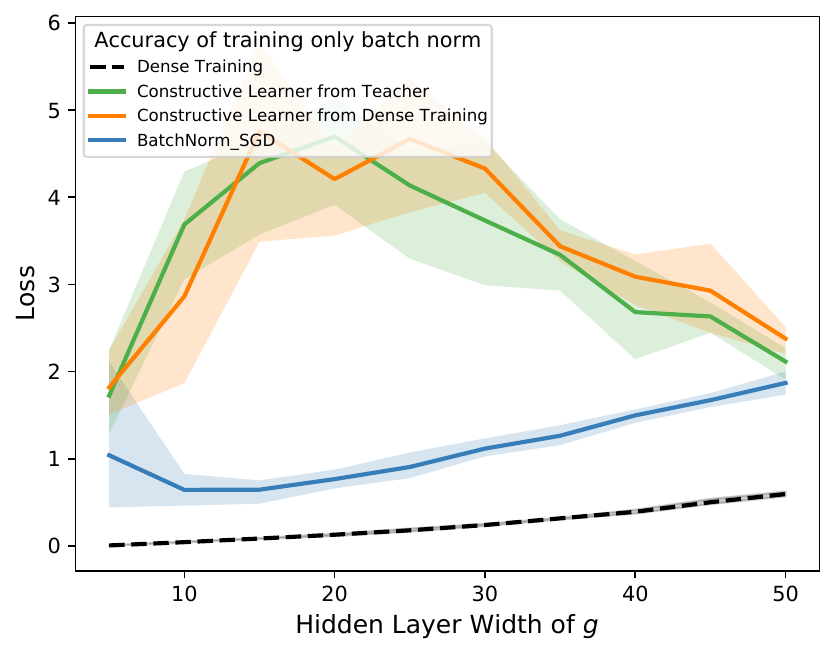}
        \hfill
         \includegraphics[width=0.49\textwidth]{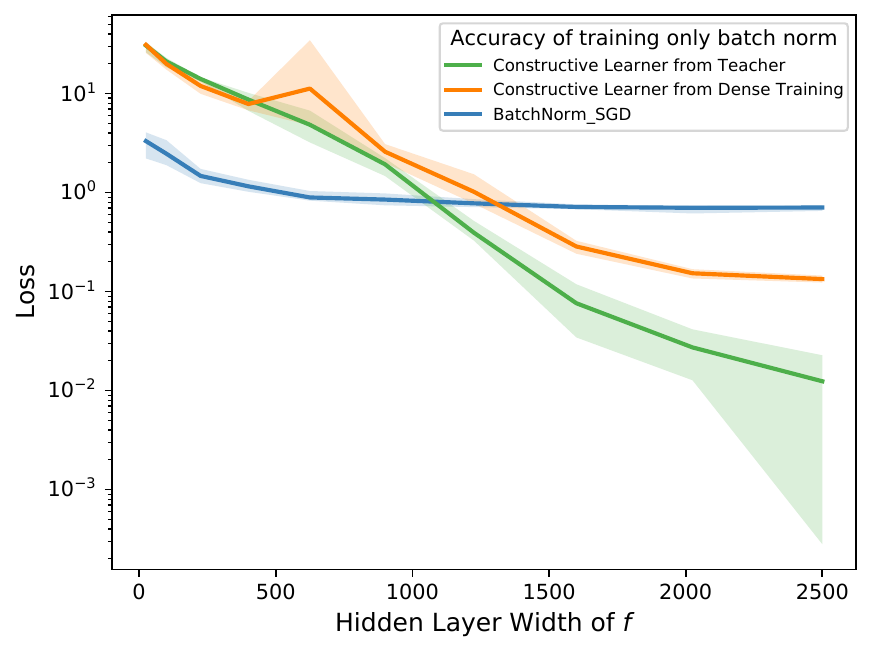}
         \caption{Sparsity = 0.05 (Weight matries have 5\% non-zero weights)}
     \end{subfigure}
    \caption{For the same setting as in Figure \ref{fig:exp1}, we set the sparsity of the frozen weights of $f_1,f_2,$ and $f_3$ to be 20\% and 5\%. The figures above show the results for these experiments. (Bottom left) We see that if the network is very sparse, and the width is also small, then SGD beats our construction.}
    \label{fig:exp2}  
\end{figure}

We tested the three algorithms when the networks $f_1,f_2$, and $f_3$ were constrained to have sparse weights. The results are shown in Figure \ref{fig:exp2}. In the plot on the left in Subfigure \ref{fig:exp2a}, we see that when the dimensions are low as compared to the sparsity, the orange and the black curve are far; as well as the green curve has high error. However, as we increase the dimensions keeping the sparsity fixed, the orange and black curves come closer and the green curve also achieves almost 0 error. This agrees with the results of Theorem \ref{th: sparse}. However, we also see that when the network is extremely sparse (sparsity=0.05), none of the algorithms achieve small loss. However, SGD (blue curve) still achieves much better loss that the constructive algorithms.

\paragraph{Experiment Details.}

The training set and the test set each consisted of 1 million samples with random Gaussian inputs and the labels were set according to the teacher network. The experiments were repeated 5 times for each method, and the smallest and the largest losses were discarded while computing the error bars for the figures. For learning \acl{BNs} params using SGD (blue curve), three learning rate schedulers were tried: Cosine annealing scheduler, exponential decay scheduler, and constant learning rate, out of which constant learning rate performed the best. The experiments were run on a machine with Intel i9-9820X CPU with 131 GB RAM and GeForce RTX 2080 Ti GPU with 11GB RAM. 
The code can be found here - \url{https://anonymous.4open.science/r/batch-norm_git-8F5B/} .

\end{document}